\documentclass[letterpaper, 10pt]{IEEEtran} 
\usepackage[letterpaper, left=0.65in, right=0.65in, bottom=0.7in, top=0.7in]{geometry} 

\usepackage{float}
\usepackage[pdftex]{graphicx}
\usepackage{fixltx2e}
\usepackage{tabularx}
\usepackage{verbatim}
\usepackage{color}
\usepackage{xparse}
\usepackage{hyperref}
\usepackage{xmpmulti}
\usepackage{transparent}
\usepackage{fancyhdr}
\usepackage{cite}
\usepackage{rotating}
\usepackage{amssymb,amsmath,float}
\usepackage[linesnumbered, algoruled, vlined]{algorithm2e}
\usepackage{booktabs}
\usepackage{caption}
\usepackage{subcaption}
\usepackage{tikz}
\usepackage{amsthm}
\usepackage{hyperref}
\usepackage{wrapfig}
\usepackage{setspace}
\usepackage{graphicx}
\usepackage{epsfig}
\usepackage[all]{xy}
\usepackage{verbatim}
\usepackage{float}
\usepackage{mathrsfs}  
\usepackage{bm}
\usepackage{multirow}
\usepackage{color,soul}
\usepackage{amssymb,amsmath,graphicx}
\usepackage{amsthm}
\usepackage[compact]{titlesec}
\usepackage{diagbox}
\usepackage{array}
\usepackage{titlesec}
\usepackage{units}
\usepackage[export]{adjustbox}

\newtheorem{theorem}{Theorem}

\newtheorem{problem}{Problem}
                       
\newtheorem{definition}{Definition}

\let\oldIEEEkeywords\IEEEkeywords
\def\IEEEkeywords{\oldIEEEkeywords\normalfont\bfseries\ignorespaces}
\usepackage{pdfpages}
\usepackage{mathtools}
\mathtoolsset{showonlyrefs}

\titleformat{\section}{\centering\normalfont\scshape}{\thesection}{0em}{.~}
\titleformat{\paragraph}[runin]{\bfseries}{}{-2em}{}[:]

\newcommand{\SampleStyle}[1]{\ensuremath{ \mathcal{#1} }}
\newcommand{\Sample}{\ensuremath{\SampleStyle{D}}}
\newcommand{\IntervalSample}{\ensuremath{\Sample_{unc}}}
\newcommand{\TracesStyle}[1]{\ensuremath{ \mathit{#1} }}
\newcommand{\Traces}{\ensuremath{\TracesStyle{Z}}}
\NewDocumentCommand{\traceStyle}{mO{}O{}}{\ensuremath{\mathit{#1}_{#2}^{#3}}} 
\NewDocumentCommand{\trace}{O{}O{}}{\ensuremath{\traceStyle{\zeta}[#1][#2]}} 
\NewDocumentCommand{\intervalStyle}{mO{}O{}}{\ensuremath{[\underline{#1}_{#2}^{#3},\overline{#1}_{#2}^{#3}]}} 
\NewDocumentCommand{\intervaltrace}{O{}O{}}{\ensuremath{\intervalStyle{\trace}[#1][#2]}} 

\newcommand{\formula}{\ensuremath{\varphi}}

\DeclareMathOperator{\ltrue}{\mathit{True}}
\DeclareMathOperator{\lfalse}{\mathit{False}}
\DeclareMathOperator*{\biglor}{\bigvee}
\DeclareMathOperator*{\bigland}{\bigwedge}
\DeclareMathOperator{\limplies}{\rightarrow}

\DeclareMathOperator{\luntil}{\mathbf{U}}
\DeclareMathOperator{\leventually}{\mathbf{F}}
\DeclareMathOperator{\lglobally}{\mathbf{G}}

\newcommand{\SAT}{SAT} 

\newcommand{\SMT}{SMT} 
\newcommand{\OMT}{OMT} 
\newcommand{\MaxSMT}{MaxSMT}
\newcommand{\OptSMT}{OptSMT}

\newcommand{\DT}{\ensuremath{\tau}}

\newcommand{\stopCriteria}{\ensuremath{\mathit{stop}}}

\newcommand{\leaf}{\ensuremath{\mathit{leaf}}}

\newcommand{\valFuncPos}[3]{\ensuremath{V({#1},{#2},{#3})}}

\newcommand{\y}[3]{\ensuremath{y_{#1,#2}^{#3}}}

\newcommand{\semanticConst}[2]{\ensuremath{\Phi^{#1}_{#2}}}
\newcommand{\specDepth}{\ensuremath{n}}

\newcommand{\propFormula}{\Phi}
\newcommand{\ObjectiveSample}{\ensuremath{F}}
\newcommand{\ObjectiveOneTraj}{\ensuremath{\Tilde{\ObjectiveSample}}}
\newcommand{\class}{\ensuremath{l}} 
\newcommand{\posclass}{\ensuremath{+1}}
\newcommand{\negclass}{\ensuremath{-1}}
\newcommand{\model}{\ensuremath{v}} 

\newcommand{\robustness}{\ensuremath{r}} 

\SetKwInput{Hyperparameter}{Hyperparameter}
\SetKwInput{Input}{Input}
\SetKwInput{Output}{Output}
\SetKwBlock{Loop}{loop}{end}

\SetKwInOut{Parameter}{Parameter}
\SetKwFunction{AlgoLearnLTL}{\text{Learn-LTL}}
\newcommand{\DTree}{\text{Decision-tree}}
\SetKwFunction{StopCriteria}{\text{Stopping-Criteria}}
\SetKwFunction{AlgoDTree}{\DTree}
\newcommand{\SplitSample}{\text{Split-sample}}
\SetKwFunction{AlgoSplitSample}{\SplitSample}
\newcommand{\Optimize}{\text{Infer-formula}}
\SetKwFunction{AlgoOptimize}{\Optimize}
\newcommand{\OptimizeMaxedSize}{\text{Infer-formula-max}}
\SetKwFunction{AlgoOptimizeMaxedSize}{\OptimizeMaxedSize}
\newcommand{\OptimizeSufficientScore}{\text{Infer-formula-min}}
\SetKwFunction{AlgoOptimizeSufficientScore}{\OptimizeSufficientScore}

\newcommand{\RobustMaxSMT}{\emph{TLI-UA}}
\newcommand{\RobustMaxSMTDT}{\emph{TLI-UA-DT}}
\newcommand{\SamplingMaxSMT}{\emph{TLI-RS}}
\newcommand{\SamplingMaxSMTDT}{\emph{TLI-RS-DT}}

\newcommand{\minclassification}{\ensuremath{S}}
\newcommand{\minrobustness}{\ensuremath{R}}
\newcommand{\maxiteration}{\ensuremath{N}}

\newcommand{\abs}[1]{\ensuremath{|#1|}}
\newcommand{\satisfies}{\ensuremath{\vDash}}
\newcommand{\real}{\ensuremath{\mathbb{R}}}

\newcommand{\nat}{\ensuremath{\mathbb{N}}}

\begin{document}
\pagenumbering{arabic}
	
	\title{Uncertainty-Aware Signal Temporal Logic Inference}
	
 	\author{
 	    Nasim Baharisangari,
 	    Jean-Raphaël Gaglione,
 	    Daniel Neider,
 	    Ufuk Topcu,
 	    Zhe Xu\thanks{Nasim Baharisangari and Zhe Xu are with the School for Engineering of Matter, Transport and Energy, Arizona State University, Tempe, AZ 85287. Jean-Raphaël Gaglione is with Ecole Polytechnique, Palaiseau, France. Daniel Neider is with Max Planck Institute for Software Systems, Kaiserslautern, Germany. Ufuk Topcu is with the Department of Aerospace Engineering and Engineering Mechanics, and the Oden Institute for Computational Engineering and Sciences, University of Texas, Austin, 201 E 24th St, Austin, TX 78712. Email: {\tt\small $\{$nbaharis, xzhe1$\}$@asu.edu, jr.gaglione@yahoo.fr, neider@mpi-sws.org, utopcu@utexas.edu}.}
 	}


	\maketitle
	
\begin{abstract} 
Temporal logic inference is the process of extracting formal descriptions of system behaviors from data in the form of temporal logic formulas. The existing temporal logic inference methods mostly neglect uncertainties in the data, which results in limited applicability of such methods in real-world deployments. In this paper, we first investigate the uncertainties associated with trajectories of a system and represent such uncertainties in the form of \textit{interval trajectories}. We then propose two \textit{uncertainty-aware} signal temporal logic (STL) inference approaches to classify the undesired behaviors and desired behaviors of a system. Instead of classifying finitely many trajectories, we classify infinitely many trajectories within the interval trajectories. In the first approach, we incorporate robust semantics of STL formulas with respect to an interval trajectory to quantify the margin at which an STL formula is satisfied or violated by the interval trajectory. The second approach relies on the first learning algorithm and exploits the decision tree to infer STL formulas to classify behaviors of a given system. The proposed approaches also work for non-separable data by optimizing the \textit{worst-case robustness} in inferring an STL formula. Finally, we evaluate the performance of the proposed algorithms in two case studies, where the proposed algorithms show reductions in the computation time by up to four orders of magnitude in comparison with the sampling-based baseline algorithms (for a dataset with 800 sampled trajectories in total).     
\end{abstract}     

\section{Introduction}
\label{sec_intro}
There is a growing emergence of the artificial intelligence (AI) in different fields, such as traffic prediction and transportation \cite{Essien2019} \cite{Essien2020} \cite{Boukerche2020}, or image \cite{Fujiyoshi2019} and pattern  recognition \cite{Sarker2021} \cite{pattern}.
Competency-awareness is an advantageous capability that can raise the accountability of AI systems \cite{Sintov2020}.
For example, the AI systems should be able to explain their behaviors in an interpretable way.

One option to represent system behaviors is to use formal languages such as linear temporal logic (LTL) \cite{Shvo2020}\cite{Basudhar2008}. LTL resembles natural language which is interpretable by humans, and at the same time preserves the rigor of formal logics. However, LTL is used for analyzing discrete programs; hence, LTL falls short when coping with continuous systems such as mixed-signal circuits in \textit{cyber-physical systems} (CPS) \cite{Raman2015}\cite{Bae2019}\cite{Raman2015}.
Signal temporal logic (STL) is an extension of LTL that tackles this issue.
STL is a temporal logic defined over signals \cite{Asarin2012}\cite{Maler2004}, and branches out LTL in two directions: on top of using atomic predicates, it deploys real-valued variables, and it is defined over time intervals \cite{Budde2018}.

Recently, inferring temporal properties of a system from its trajectories have been in the spotlight.
We can derive such properties by inferring STL formulas.
One way of computing such STL formulas is incorporating \textit{satisfibilty modulo theories} (\SMT{}) solvers such as \textit{Z3 theorem prover} \cite{10.5555/1792734.1792766}.
Such solvers can determine whether an STL formula is satisfiable or not, according to some theories such as \textit{arithmetic theory} or \textit{bit-vector theory} \cite{DeMoura2009}\cite{Clarke2018}.
Most of the frequently used algorithms for STL inference deploy customized templates for a solver to conform to.
These pre-determined templates suffer from applying many limitations on the inference process.
First, a template that complies with the behavior of the system is not easy to handcraft.
Second, the templates can constrain the structure of an inferred formula; hence, the formulas not compatible with the template are removed from the search space \cite{Neider2019}\cite{Bombara2016}\cite{Xu2016}\cite{Xu2015}.

Moreover, the importance of \textit{uncertainties} is well established in several fields such as machine learning (ML) application in weather forecast \cite{Moosavi2021}, or deep learning application in facilitating statistical inference \cite{Malinin2019}. 
Failure to account for the uncertainties in any system can reduce the reliability of the output as \textit{predictions} in AI systems \cite{Hubschneider2019}. Specifically, in AI, uncertainties happen due to different reasons including noise in data or overlap between classes of data \cite{Hubschneider2019}. \textit{Uncertainty quantification} techniques in AI systems have been developed to address this matter \cite{Abdar2020}. Incorporating the quantified \textit{uncertainties} in the process of inferring temporal properties can lead to higher credibility of the predictions. However, most of the current approaches of inferring temporal properties do not account for uncertainties in the inference process. \par
In this paper, we propose two \textit{uncertainty-aware} algorithms in the form of STL formulas for inferring the temporal properties of a system. In the proposed algorithms, the uncertainties associated with trajectories describing the evolution of a system is implemented in the form of \textit{interval trajectories}. The second algorithm relies on the first one and uses decision tree to infer STL formulas. By taking uncertainties into consideration, instead of classifying finitely many trajectories to infer an STL formula that classifies the trajectories, we classify infinitely many trajectories within the interval trajectories (Section \ref{problem formulation}). See Figure \ref{uuu} as an illustration. \par
     \begin{figure}[ht]
    \centering
     \begin{subfigure}{0.23\textwidth}
         \centering
         \includegraphics[scale=0.17]{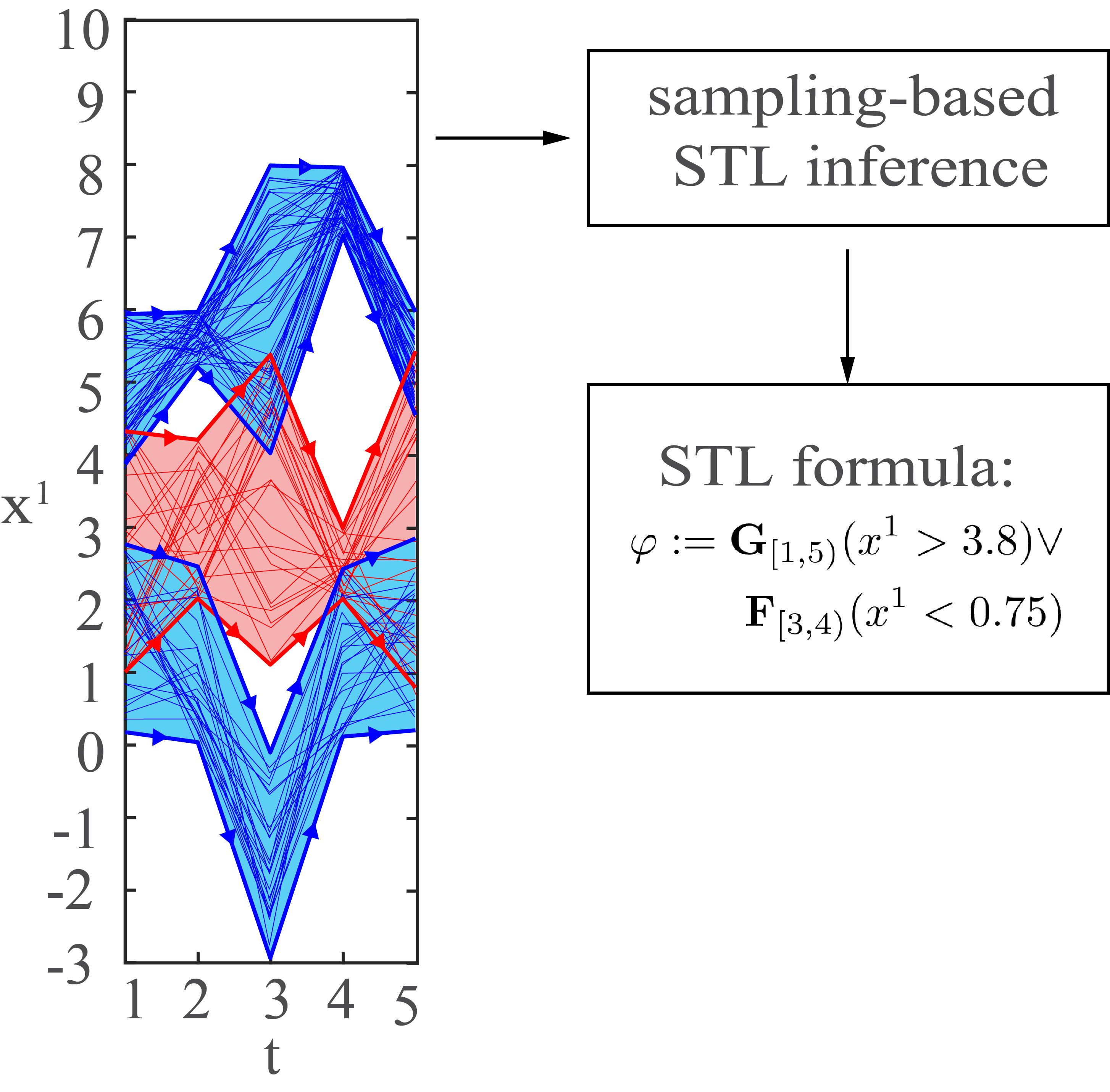}
         \caption{}
         \label{s}
     \end{subfigure}
     \begin{subfigure}{0.23\textwidth}
         \centering
         \includegraphics[scale=0.17]{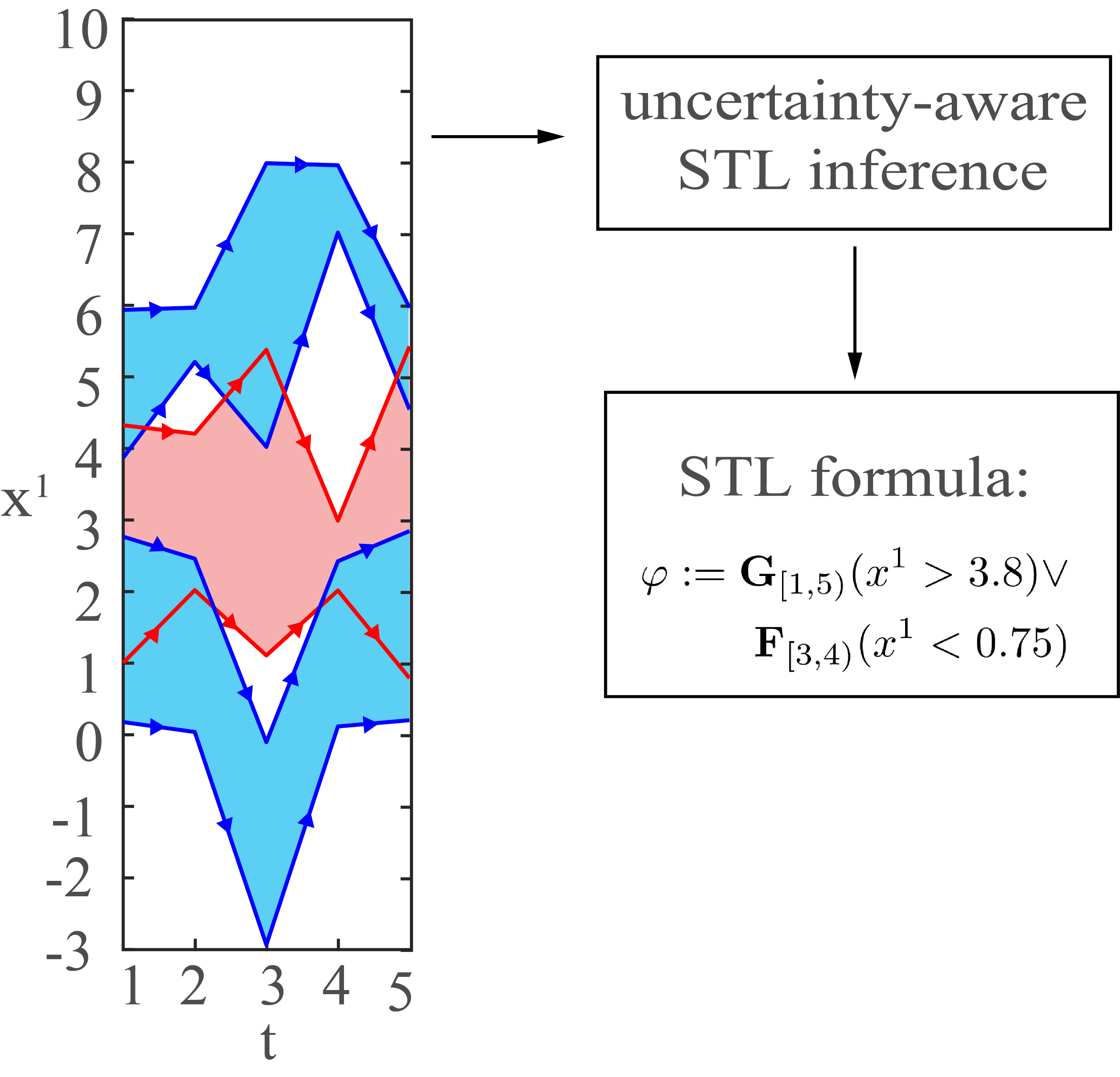}
         \caption{}
         \label{b}
     \end{subfigure}
     \caption{This figure shows a dataset with one sate $x^1$. (a) For inferring an STL formula by sampling-based baseline algorithms, we sample finitely many trajectories from each interval trajectory to infer an STL formula to describe the behavior of the trajectories. (b) By uncertainty-aware STL inference, we infer an STL formula to describe the temporal properties of the infinitely many trajectories within the interval trajectories. In this example, the blue tubes represent the interval trajectories with the desired behavior, and the red tube represents the interval trajectory with the undesired behavior. By both the uncertainty-aware STL and sampling-based baseline algorithms, we infer the STL formula $\lglobally_{[1,5)}(x^1>3.8)\lor{\leventually}_{[3,4)}(x^1<0.75)$ that perfectly classifies the interval trajectories with the desired behavior (blue tube) and the interval trajectories with the undesired behavior (red tube). }\label{uuu}
\end{figure}

Besides, the uncertainty-aware algorithms quantify the margin at which an interval trajectory violates or satisfies an inferred STL formula by incorporating the robust semantics of uncertainty-aware STL formulas (Section \ref{STL formula robustness margin}).\par
In addition, we maximize \textit{worst-case robustness} margin for increasing the robustness of the inferred STL formulas in the presence of uncertainties; thus, we can view our problem as an optimization problem with the objective of maximizing the worst-case robustness margin. Moreover, we introduce a framework in which we employ the optimized robustness to compute the optimal STL formula for non-separable datasets (Section \ref{Classifying non-separable interval trajectories framework}). We conduct experiments to evaluate the performance of the proposed learning algorithms. Our experimental results show that inferring STL formulas using \textit{interval trajectories} reduces the computation time up to four orders of magnitudes in comparison with sampling-based baseline algorithms (for a dataset with 800 sampled trajectories in total) while the inferred formula by uncertainty-aware STL is robust in the presence of uncertainties. Incorporating decision tree in the second proposed algorithm further expedites the process by having the computation time at most $1/88$ of the computation time of the sampling-based baseline algorithms with decision trees (for a dataset with 800 sampled trajectories in total) (Section \ref{experimental evaluation}). \par

\paragraph{Related Work}
\label{related work}
Recently, inferring temporal properties of a system by its temporally evolved trajectories have been in the spotlight \cite{Jin2015}\cite{Jha2017}\cite{Vazquez-Chanlatte2017}. Kong \textit{et. al} developed an algorithm for STL inference from trajectories that show the desired behavior and those having the undesired behavior \cite{Kong2014}. Kyriakis \textit{et. al.} proposed \textit{stochastic temporal logic} (StTL) inference algorithm for characterizing time-varying behaviors of a dynamical stochastic system. Neider \textit{et. al} \cite{Neider2019} proposed a framework for learning temporal properties using LTL. Several other frameworks for inferring temporal properties such as \cite{Bombara2018}, \cite{Nguyen2017}, \cite{Bombara2016}, \cite{Xu2015}, \cite{akazaki2015}, etc have been proposed.\par
In addition, various quantitative semantics have been introduced for temporal logics such as robust semantics. Danze \textit{et. al.} \cite{Maler2010} introduced robust semantics for STL formulas, and Lindemann \textit{et. al.} \cite{Lindemann2019} introduced predicate robustness. Fainekos \textit{et. al.} \cite{Fainekos2009} defined robust semantics for \textit{metric temporal logic} (MTL). Kyriakis \textit{et. al.} incorporated robust semantics to maximize the efficiency of his StTL inference framework \cite{Kyriakis2019}. On the other hand, using decision trees in STL inference have been noteworthy \cite{Nebut2004}\cite{bartocco2014}\cite{Bombara2016}\cite{Bufo2014}\cite{Nguyen2017}. For example, Brunello \textit{et. al.} \cite{Brunello2019} introduced \textit{Interval Temporal Logic Decision Trees}, where the data is delivered to the decision tree in form of intervals.\par


\section{Preliminaries}
\label{Preliminaries}

In this section, we set up definitions and notations used throughout this paper.

\paragraph{Finite trajectories}
We can describe the state of an underlying system by a vector $x=[x^{1},x^{2},..,x^{n}]$, where $n$ is a non-negative integer (the superscript $i$ in $x^i$ refers to the $i$-th dimension). The domain of $x$ is denoted by $\mathbb{X}=\mathbb{X}^1\times\mathbb{X}^2\times...\times\mathbb{X}^n$, where each $\mathbb{X}^i$ is a subset of $\real$.
The evolution of the underlying system within a finite time horizon is defined in the discrete time domain $\mathbb{T}=\{t_0,t_1,...,t_J\}$, where $J$ is a non-negative integer. We define a finite \textit{trajectory} describing the evolution of the underlying system as a function $\trace:\mathbb{T}\rightarrow\mathbb{X}$. We use $\trace[j]\triangleq{x(t_j)}$ to denote the value of $\trace$ at time-step $t_j$.\par
\paragraph{Intervals and interval trajectories}
An \textit{interval}, denoted by $\intervalStyle{a}$,
is defined as $\intervalStyle{a}:= \left\{ a\in\real^n \middle| \underline{a}^i \leq a^i \leq \overline{a}^i, i=1,..,n \right\} $, where $\underline{a},\overline{a}\in\real^n$, and $\underline{a}^i\leq\overline{a}^i$ holds true for all $i$. The superscript $i$ refers to the $i$-th dimension. For the purpose of this work, we introduce
\textit{interval trajectories}. We define an \textit{interval trajectory} $\intervaltrace$ as a set of trajectories such that for any $\trace\in\intervaltrace$, we have $\trace[j]\in \intervaltrace[j]$ for all $t_j\in\mathbb{T}$ \cite{Xu2021}. We know that the time length of a trajectory $\zeta\in\intervalStyle{\zeta}$ is equal to the time length of an interval trajectory $\intervalStyle{\zeta}$; thus, we can denote the time length of an interval trajectory $\intervalStyle{\zeta}$ with $|\zeta|$.
\section{Signal temporal logic and robust semantics of interval trajectories}

\subsection{~Signal Temporal Logic}
\label{signal temporal logic}
We first briefly review the signal temporal logic (STL).
We start with the Boolean semantics of STL.
The domain $\mathbb{B}=\{True,False\}$ is the Boolean domain.
Moreover, we introduce a set $\varPi=\{\pi_{1},\pi
_{2},\dots,\pi_{n}\}$ which is a set of predefined \textit{atomic predicates}.
Each of these predicates can hold values \textit{True} or \textit{False}.
The syntax of STL is defined recursively as follows.
\begin{align*}
    \formula &:=
    \top
    \mid \pi
    \mid \lnot\formula
    \mid \formula_{1}\land\formula_{2}
    \mid \formula_{1}\lor\formula_{2}
    \mid \formula_{1}\luntil_{I}\formula_{2}
\end{align*}

\noindent where $\top$ stands for the Boolean constant \textit{True},
$\pi$ is an atomic predicate in the form of an inequality $f(x) > 0$ where $f$ is some real-valued function.
$\lnot$ (negation), $\land$ (conjunction), $\lor$ (disjunction) 
are standard Boolean connectives, and ``$\luntil$'' is the temporal operator ``until''.
We add syntactic sugar, and introduce the temporal operators ``$\leventually$'' and ``$\lglobally$'' representing ``{eventually}'' and ``always'', respectively.
$I$ is a time interval of the form $I=[a,b)$, where $\textcolor{black}{a < b}$, and they are non-negative integers. We call the set containing all the mentioned operators $C= \left\{ \top,\land,\lor,\lnot,\limplies,\leventually,\lglobally,\luntil \right\}$.

We employ the Boolean semantics of an STL formula in \textit{strong} and \textit{weak} views. We denote the length of a trajectory $\trace \in \intervaltrace$ by $T$ \cite{Xu2019}.
In the strong view, $(\trace,t_j)\models_{S}\formula$ means that trajectory $\trace\in\intervaltrace$ strongly satisfies the formula $\formula$ at time-step $t_j$.
$(\trace,t_j)\not\models_{S}\formula$ means that trajectory 
$\trace\in\intervaltrace$ strongly violates the formula 
$\formula$ at time-step $t_j$. Similar interpretations hold true for the weak view ($\models_{W}$ means "satisfies weakly" and $\not\models_{W}$ means "violates weakly"~). By considering the Boolean semantics of STL formulas in strong and weak views, the strong satisfaction or violation of an STL formula $\formula$ by a trajectory at time-step $t_j$ implies the weak satisfaction or violation of the STL formula by the trajectory at time-step $t_j$. Consequently, we can take either of the views for trajectories with label $\class_i=\posclass$ and trajectories with label $\class_i=\negclass$ for perfect classification \cite{Xu2019}. In this paper, we choose to adopt the strong view for the problem formulations of Sections \ref{problem formulation} and \ref{Classifying non-separable interval trajectories framework}. 
\begin{definition}
The Boolean semantics of an STL formula $\formula$, for a trajectory $\trace\in\intervaltrace$ with the time length of $T$ at time-step $t_j$ in strong view is defined recursively as follows.
\[
\begin{split}
(\trace,t_j)\models_{S}\pi~\mbox{iff}&~~
t_j\leq{T}~\mbox{and}~
{f}(\trace[j])>0\\
(\trace,t_j)\models_{S}\lnot\formula~\mbox{iff}&~~(\trace,t_j)\not\models_{W}\formula,\\
(\trace,t_j)\models_{S}\formula_{1}\wedge\formula_{2}~\mbox{iff}&~~(\trace,t_j)\models_{S}\formula_{1}~\mbox{and}\\
&~~(\trace,t_j)\models_{S}\formula_{2},\\
(\trace,t_j)\models_{S}\formula{_1}\luntil_{[a,b)}\formula_2~\mbox{iff}&~~\exists{j'}\in[j+a,j+b),\\
~(\trace,t_{j'}) \models_{S} \formula_2 ~\mbox{and} &~~\forall{j''}\in{[j+a,j')}\mbox{,}
~(\trace,t_{j''}) \models_{S} \formula_1.
\end{split}
\]
\end{definition}
               
\begin{definition}
 The Boolean semantics of an STL formula $\formula$, for a trajectory $\trace\in\intervaltrace$ with the time length of $T$ at time-step $t_j$ in weak view is defined recursively as follows.
\begin{align*}
(\trace,t_j)\models_{W}\pi~\mbox{iff}&~~
\mbox{either}~\mbox{of}~\mbox{the}~\mbox{1)}~\mbox{or}~\mbox{2)}~\mbox{ holds:}~~~~~~~~~~~\\
&\mbox{1) }t_j>T,~
\mbox{2) }t_j\leq{T}~\mbox{and}\\&~~
f(\trace[j])>0,\\
(\trace,t_j)\models_{W}\lnot\formula~\mbox{iff}&~~ (\trace,t_j)\not\models_{S}\formula,\\
(\trace,t_j)\models_{W}\formula_{1}\wedge\formula_{2}~\mbox{iff}&~~  (\trace,t_j)\models_{W}\formula
_{1}~\mbox{and}\\
&~~(\trace,t_j)\models_{W}\formula_{2},\\
(\trace,t_j)\models_{W}\formula{_1}\luntil_{[a,b)}\formula_2~\mbox{iff}&~~\exists{j'}\in[j+a,j+b),\\
~(\trace,t_{j'}) \models_{W} \formula_2 ~\mbox{and} &~~\forall{j''}\in{[j+a,j')}\mbox{,}
~(\trace,t_{j''}) \models_{W} \formula_1.
\end{align*}
\end{definition}

\paragraph{Syntax DAG}
Any STL formula can be represented as a syntax directed acyclic graph, i.e., syntax DAG. In a syntax DAG, the nodes are labeled with atomic predicates or temporal operators that form an STL formula \cite{Neider2019}.
For instance, Figure \ref{example DAG} shows the unique syntax DAG of the formula $(\pi_1\luntil\pi_2) \land \lglobally(\pi_1\lor\pi_2)$, in which the subformula $\pi_2$ is shared.
Figure \ref{example DAG indetifiers} shows arrangement of the identifiers of each node in the syntax DAG ($i\in\{1,..,7\}$).

\begin{figure}[ht]
    \centering
     \begin{subfigure}{0.23\textwidth}
         \centering
         \includegraphics[scale=0.2]{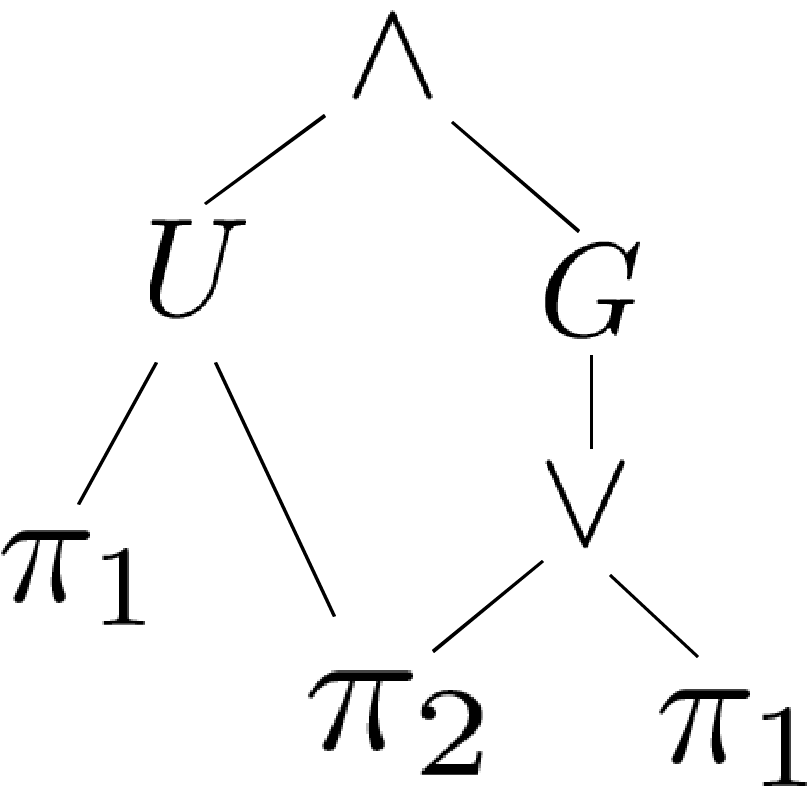}
         \caption{Syntax DAG}
         \label{example DAG}
     \end{subfigure}
     \begin{subfigure}{0.23\textwidth}
         \centering
         \includegraphics[scale=0.2]{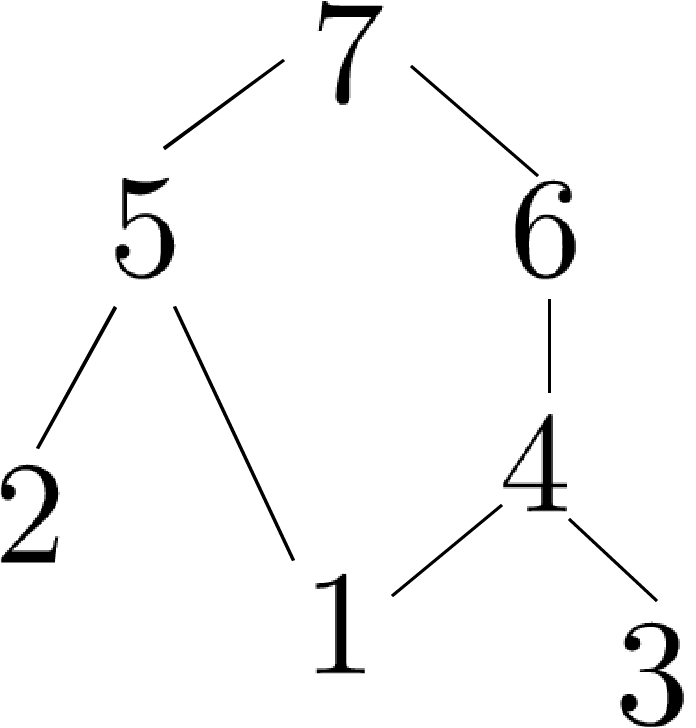}
         \caption{Identifiers}
         \label{example DAG indetifiers}
     \end{subfigure}
     \caption{Syntax DAG and identifier of syntax DAG of the formula $(\pi_1\luntil\pi_2) \wedge \lglobally(\pi_1\vee\pi_2)$}
\end{figure}

\vspace{-20pt}
\paragraph{Size of an STL formula}
If we present an STL formula by a syntax DAG, then each node corresponds to a subformula; thus, the size of an STL formula is equal to the number of the DAG nodes. We denote the size of an STL formula $\formula$ by $|\formula|$ \cite{Schneider2004}.

\subsection{~Robust Semantics of STL Formulas }\label{STL formula robustness margin}

Robust semantics quantifies the margin at which a certain trajectory satisfies or violates an STL formula $\formula$ at time-step $t_j$. The robustness margin of a trajectory $\trace$ with respect to an STL formula
$\formula$ at time-step $t_j$ is given by $\robustness(\trace,\formula,t_j)$, where $\robustness(\zeta,\formula,t_j)$ can be calculated
recursively via the robust semantics \cite{Fainekos2009}.

\begin{align*}
    \begin{split}
    \robustness(\trace,\pi,t_j)  &= f(\trace[j]),\\
    \robustness(\trace,\lnot\formula,t_j)  &= -\robustness(\trace,\formula,t_j),\\
    \robustness(\trace,\formula_{1}\wedge\formula_{2},t_j)  &= \min(\robustness(\trace,\formula_{1},t_j),\robustness(\trace,\formula_{2},t_j)),\\
    \robustness(\trace,\formula_{1}\luntil_{\lbrack{a},b)}\formula_{2},t_j) &=
    \max\limits_{j+a \leq j' < j+b}(  \min( \robustness(\trace,\formula_{2},t_{j'}), \\
    &\min\limits_{j+a \leq j'' < j'} \robustness(\trace,\formula_{1},t_{j''}) )).
    \end{split}
\end{align*}

We can define the robustness margin of an interval trajectory in two views: \textit{worst-case} and \textit{best-case}.
The worst-case view chooses the trajectory with the minimum corresponding robustness within an interval trajectory (Eq. \eqref{robustness transfomation-worst}).
The best-case view chooses the trajectory with maximum corresponding robustness within an interval trajectory (Eq. \eqref{robustness transfomation-best}); thus, we define the robustness margin of an interval trajectory in two views, as follows.
\begin{align}
    \underline{\robustness}(\intervaltrace,\formula,t_j)=\min\limits_{\trace\in\intervaltrace}\robustness(\trace,\formula,t_j)
    \label{robustness transfomation-worst}
    \\
    \overline{\robustness}(\intervaltrace,\formula,t_j)=\max\limits_{\trace\in\intervaltrace}\robustness(\trace,\formula,t_j)
    \label{robustness transfomation-best}
\end{align}

\section{~Problem formulation} \label{problem formulation}
In this section, we present the problem formulation of perfectly classifying interval trajectories. Then, we derive the sufficient conditions to solve the problem. The existing methods for inferring STL formulas mostly classify finitely many trajectories without considering the uncertainties. For the problem of classifying finitely many trajectories, we define the following. 
\begin{definition}                                                                                         
	Given a labeled set of trajectories $\Sample=\{({\trace}^i,\class_i)\}^{N_\mathcal{D}}_{i=1}$, $\class_i=\posclass$ represents the desired behavior and $\class_i=\negclass$ represents the undesired behavior, an STL formula $\formula$, evaluated at time $t_0$, perfectly classifies the desired behaviors and the undesired behaviors if the following condition is satisfied.\\
	$({\trace}^i,t_0)\models_{S}\formula$, if $\class_i=\posclass$; $({\trace}^i,t_0)\models_{S}\lnot\formula$, if $\class_i=\negclass$.
	\label{trajectory def}
\end{definition}
With Definition \ref{trajectory def}, the problem of STL inference for classifying finitely many trajectories is as follows.\par
\begin{problem}\label{problem trajectory}                                                                                     
Given a labeled set of trajectories $\Sample=\{({\trace}^i,{\class}_i)\}^{N_{\mathcal{D}}}_{i=1}$, compute an STL formula $\formula$ such that $\formula$, which is evaluated at time $t_0$, perfectly classifies the desired behaviors and undesired behaviors, and $|\formula|\leq{\maxiteration}$, where $\maxiteration$ is a predetermined positive integer.
\end{problem}
Definition \ref{trajectory def} cannot classify infinitely many trajectories; thus, by taking the uncertainties into consideration, and substituting trajectories with interval trajectories, we define the following.
\begin{definition}                                                                                      
	Given a labeled set of interval trajectories $\IntervalSample=\{(\intervaltrace^i,\class_i)\}^{N_\mathcal{D}}_{i=1}$, $\class_i=\posclass$ represents the desired behavior and $\class_i=\negclass$ represents the undesired behavior, an STL formula $\formula$, which is evaluated at time $t_0$, perfectly classifies the desired behaviors and the undesired behaviors if the following condition is satisfied.\\
	if $\class_i=\posclass$, then $\forall\trace\in\intervaltrace^i$, we have $(\trace,t_0)\models_{S}\formula$; 
	if $\class_i=\negclass$, then $\forall \trace\in\intervaltrace^i$, we have $(\trace,t_0)\models_{S}\lnot\formula$.
	\label{interval trajectory def}
	\end{definition}
	Now, we define a problem formulation of classifying infinitely many trajectories within the interval trajectories.
\begin{problem}\label{problem interval trajectory} 

Given a labeled set of interval trajectories $\IntervalSample=\{(\intervaltrace^i,\class_i)\}^{N_\mathcal{D}}_{i=1}$, compute an STL formula $\formula$, which is evaluated at time $t_0$, perfectly classifies the desired behaviors and undesired behaviors, and $|\formula|\leq{\maxiteration}$, where $\maxiteration$ is a predetermined positive integer.
\end{problem}
We need a \textit{sufficient condition} that allows us to use Definition \ref{interval trajectory def} to perfectly classify interval trajectories. Before deriving a such condition, we define \textit{separable} interval trajectories as follows.\par

\begin{definition}
    \label{def:separable}
    
    We define that two interval trajectories $\intervaltrace$ and $\intervaltrace'$ are separable
    if there exists at least one time-step $t_j$ and one dimension $k$ such that the two intervals $\intervaltrace[j][k]$ and $\intervaltrace[j][k]'$ do not intersect, i.e., $\intervaltrace[j][k] \cap \intervaltrace[j][k]' = \emptyset$.
\end{definition}  
\begin{definition}\label{def:seperable sets}
    We define that two finite sets of interval trajectories $\mathrm{Z}$ and $\mathrm{Z}'$ are separable if all pairs of interval trajectories 
    $\intervaltrace \in \mathrm{Z}$ and $\intervaltrace' \in \mathrm{Z}'$
    are separable.
    
    By extension, we write that a labeled set of interval trajectories $\IntervalSample=\{(\intervaltrace^i,\class_i)\}^{N_\mathcal{D}}_{i=1}
    $ is separable
    if $\left\{ \intervaltrace^i \middle| \class_i=\posclass \right\}$
    and $\left\{ \intervaltrace^i \middle| \class_i=\negclass \right\}$
    are separable.
\end{definition}

Now, we provide a sufficient condition that allows us to use Definition \ref{interval trajectory def} to classify two interval trajectories.
\begin{theorem}\label{thm:perfect-classification-2traj}
If $\intervaltrace^i$
with label $\class_i=\posclass$  and 
$\intervaltrace^{\Tilde{i}}$ with label $\class_{\Tilde{i}}=\negclass$ are separable, then there exists at least one STL formula that perfectly classifies these two interval trajectories.
\end{theorem}
\begin{proof}
 See Appendix \ref{apdx:prf:perfect-classification-2traj}
\end{proof}
Now that we have the sufficient condition for perfect classification of two interval trajectories, we provide the sufficient condition for the case of having multiple interval trajectories.\par 
\begin{theorem}\label{thm:perfect-classification-sample}
  If a given labeled set of interval trajectories $\IntervalSample=\{(\intervaltrace^i,\class_i)\}^{N_\mathcal{D}}_{i=1}$ is separable, then
  there exists at least one STL formula $\formula$ that perfectly classifies $\IntervalSample$.
\end{theorem}
\begin{proof}
 See Appendix \ref{apdx:prf:perfect-classification-sample}
\end{proof}

\section{The framework of STL inference for non-separable interval trajectories}
\label{Classifying non-separable interval trajectories framework}
One source of uncertainty in a dataset is overlap between the interval trajectories satisfying an STL formula $\formula$ and interval trajectories violating $\formula$. This type of dataset is called \textit{non-separable} dataset \cite{Jiang2014}. We deploy robust semantics to set up a method to infer STL formulas for non-separable dataset with two labeled classes. Therefore, we define that two interval trajectories are \textit{non-separable} if they are not separable according to Definition \ref{def:separable}. Similarly, we define a \text{non-separable} labeled dataset if it is not separable according to Definition \ref{def:seperable sets}.\par

Given a set of $N_D$ labeled interval trajectories $\IntervalSample=\{(\intervaltrace^i,\class_i)\}^{N_\mathcal{D}}_{i=1}$,
we define in Eq. \eqref{eq:objective-onetraj} a function $\ObjectiveOneTraj$ that gives the worst-case robustness margin of an interval trajectory $\intervaltrace^i$ with respect to $\formula$ or $\lnot\formula$ if $\class_i=\posclass$ or $\class_i=\negclass$, respectively.
\begin{align}
 {\ObjectiveOneTraj(\intervaltrace^i,\class_i,\formula)}
 &:=
  \begin{cases}
   \underline{r}(\intervaltrace^i,\formula,t_0), & \text{if $\class_i=\posclass$}.\\
  \underline{r}(\intervaltrace^{i},\lnot\formula,t_ 0) , & \text{if $\class_{i}=\negclass$}.
 \end{cases}
  \label{eq:objective-onetraj}
\end{align}

We then construct in Eq. \eqref{eq:objective-sample} our objective function $\ObjectiveSample$.
If we consider the STL formula for perfect classification of $\IntervalSample$: 
$
\bigland_{\trace\in\intervaltrace^i, \class_i=\posclass}
\left( \trace \models_{S} \formula \right)
\land
\bigland_{\trace\in\intervaltrace^i, \class_i=\negclass}
\left( \trace \models_{S} \lnot\formula \right)
$,
$\ObjectiveSample$ would be the worst-case robustness margin of it.
Hence, $\ObjectiveSample$ represents the lower worst-case robustness margin amongst all the interval trajectories.
\begin{align}
  \ObjectiveSample(\IntervalSample,\formula)
  &:=
  \min\limits_{i=1,..,N_D}{\ObjectiveOneTraj(\intervaltrace^i,\class_i,\formula)}
  \label{eq:objective-sample}
\end{align}

\begin{problem}\label{non-separable data prob 1}
    Given a possibly non-separable set of labeled interval trajectories $\IntervalSample=\{(\intervaltrace^i,\class_i)\}^{N_\mathcal{D}}_{i=1}$, compute an STL formula $\formula$ that maximizes $\ObjectiveSample(\IntervalSample,\formula)$ such that $|\formula|\leq{\maxiteration}$, where $\maxiteration$ is a predetermined positive integer. 
\end{problem}
 To solve Problem \ref{non-separable data prob 1}, we compute an STL formula $\formula$ by maximizing \ObjectiveSample$(\IntervalSample,\formula)$, and we set an upper-bound on the size of the  $\formula$ for interpretability.\par
 Figure \ref{non-separable example} shows a simple illustrative example of how we use \ObjectiveOneTraj$(\intervaltrace^i,\class_i,\formula)$~ for computing such an STL formula.
In this example, the inferred STL formula can be in the form of $\formula := x^1>c$, where $x^1$ is the state of a 1-dimensional trajectory.

\begin{figure}[h]
    \centering
    \includegraphics[scale=0.4]{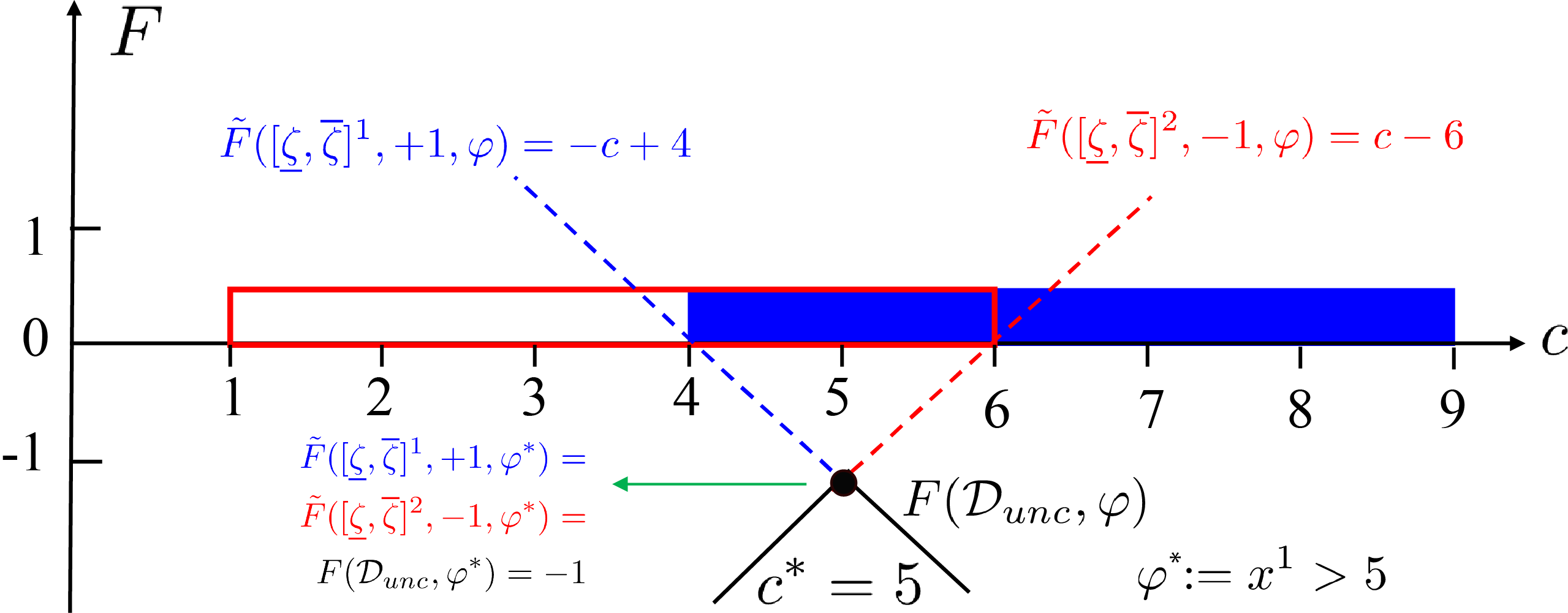}
    \caption{
        In this illustrative example we show how we use $\ObjectiveSample(\IntervalSample,\formula)$ and $\ObjectiveOneTraj(\intervaltrace^i,\class_i,\formula)$ to compute an STL formula for non-separable interval trajectories.
        We consider a labeled set $\IntervalSample$ composed of two interval trajectories $\intervaltrace^1$ and $\intervaltrace^2$,
        and a unique atomic predicate $\pi_1 := f(x) > 0$ with $f(x) = x^1 - c$, where $c$ is an unknown constant.
        For sake of simplicity, we only consider STL formula $\formula$ in the form of $\pi_1$ in this example.
        The first interval trajectory is defined such that $\class_1 = \posclass$ and $\intervaltrace[0][1]^1 = [4,9]$ (filled blue box).
        The second interval trajectory is defined such that $\class_2 = \negclass$ and $\intervaltrace[0][1]^2 = [1,6]$ (empty red box).
        The vertical axis represents the objective function $\ObjectiveSample$, as well as the underlying functions $\ObjectiveOneTraj$ for each interval trajectory,
        and the horizontal axis represents $c$.
        In this example, the optimal value for $\ObjectiveSample(\IntervalSample, \formula)$ is $-1$ and is achieved at $\formula^{*} = x^1 > c^{*}$ and $c^{*}=5$.
    }
    \label{non-separable example}
\end{figure}

\section{Uncertainty-aware SMT-based Algorithms}\label{SMT-based algorithms}
In this section, we propose and explain two uncertainty-aware algorithms for STL inference. The first uncertainty-aware algorithm is denoted as \RobustMaxSMT{}. The second algorithm relies on the first one and uses decision tree to infer STL formulas which is denoted as \RobustMaxSMTDT{}.

\paragraph{Satisfiabilty modulo theories (\SMT{}) and \OptSMT{} solvers}
One of the common concepts involved in the process of inferring STL formulas is \textit{satisfiabilty}.
Satisfiabilty addresses the following problem statement: \textit{How can we assess whether a logical formula is satisfiable?} Based on the type of the problem we deal with, we can exploit different solvers.
For example, if the satisfiabilty problem is in the Boolean domain, we can use \textit{Boolean satisfiability} (\SAT{}) solvers; or if the satisfiability problem is in the continuous domain, we can use \textit{satisfiability modulo theories} (\SMT{}) solvers.
In this paper, we consider trajectories in the real domain, as well as time-bounds of temporal operators as integers, thus we use \SMT{} instead of \SAT{}.
\SMT{} solvers, based on some theories including \textit{arithmetic theories} or \textit{bit-vectors theory}, determine whether an STL formula is satisfiable or not \cite{DeMoura2009}.

Another important concept involved is the incorporation of \textit{optimization} procedures in a \SMT{} problem.
Optimization modulo theories (\OMT{}) is an extension of \SMT{} that allows finding models that optimizes a given objective function \cite{DBLP:journals/corr/SebastianiT17}.
There exists several variants of \OMT{} \cite{DBLP:conf/tacas/BjornerPF15},
in particular \OptSMT{}, where an explicit objective function is given,
and \MaxSMT{}, where soft constraints with different costs are assigned.
Our proposed methods rely on \OptSMT{}.


\paragraph{\OptSMT{}-based learning algorithm}
In the proposed algorithms, the input to the algorithms is in the form of interval trajectories.
This input contains the data whose behaviors we wish to infer STL formulas for, and consists of two possibly non-separable labeled sets ${P_{unc}}\mbox{, } {N_{unc}}$. 
We categorize ${P_{unc}}$ as the set containing interval trajectories with desired property (or behavior) and  ${N_{unc}}$ as the set containing interval trajectories with the undesired property (or behavior).

We represent the input as a set $\IntervalSample=P_{unc}\cup{N}_{unc}$. 

We present an STL formula by a syntax DAG. The syntax DAG encodes the STL formula by propositional variables. The propositional variables are \cite{Neider2019}:\par
\begin{itemize}
\setlength{\itemsep}{6pt}%
    \item $x_{i, \lambda}$ where $i\in\{1,..., n\}$ and $\lambda\in\Pi\cup{C}$
    \item $l_{i, k}$ where $i\in\{2,.., n\}$ and $k\in\{1,.., i-1\}$
    \item $r_{i, k}$ where $i\in\{2,.., n\}$ and $k\in\{1,.., i-1\}$
\end{itemize}\par
\vspace{5pt}
 To facilitate the encoding, a \textit{unique identifier} is assigned to each node of the DAG, which is denoted by $\textit{i}\in\{1,2,..,\textit{n}\}$. Two mandatory properties of this identifier are: 1) The identifier of the root is $n$, and 2) the identifiers of the children of Node $i$ are less than $i$.
 It should be noted that the node with identifier 1 is always labeled with an atomic predicate $( \pi\in\Pi)$.
 In the listed variables, $\textit{x}_{\textit{i},\lambda}$ is in charge of encoding a labeling system for the syntax DAG such that if a variable $\textit{x}_{\textit{i},\lambda}$ becomes $\ltrue$, then Node $i$ is labeled with $\lambda$.
 The variables $\textit{l}_{\textit{i},\textit{k}}$ and $\textit{r}_{\textit{i},\textit{k}}$ encode the left and right children of the inner nodes.
 If the variable $\textit{l}_{\textit{i},\textit{k}}$ is set to $\ltrue$, then $k$ is the identifier of the left child of Node $i$. If the variable $\textit{r}_{\textit{i},\textit{k}}$ is set to $\ltrue$, then $k$ is the identifier of the right child of Node $i$. Moreover, if Node $i$ is labeled with an unary operator, the variables $\textit{r}_{\textit{i},\textit{k}}$ are ignored and, both of the variables $\textit{l}_{\textit{i},\textit{k}}$ and $\textit{r}_{\textit{i},\textit{k}}$ are ignored in the case that Node $i$ is labeled with an atomic predicate. Moreover, the reason that identifier $i$ ranges from 2 to $n$ for variables $\textit{l}_{\textit{i},\textit{k}}$ and $\textit{r}_{\textit{i},\textit{k}}$ is that Node 1 is always labeled with an atomic predicate and cannot have children.
 Lastly, $k$ ranging from 1 to $i-1$ reflects the point of children of a node having identifiers smaller than their root node. It is crucial to guaranty that each node has exactly one right and one left child.
 This matter is enforced by Eqs. \eqref{eq1}, \eqref{eq2} and \eqref{eq3}.
 In addition, Eq. \eqref{eq4} ensures that Node 1 is labeled with an atomic predicate.

\begin{align}\label{eq1}  \left [  \bigland\limits_{1 \leq  i  \leq n} \biglor\limits_{\lambda\in{\Pi}\cup{C}}  x_{i,\lambda}  \right ] \land \left [  \bigland\limits_{1 \leq  i  \leq n} \bigland\limits_{\lambda\neq\lambda^{'}\in{\Pi}\cup{C}}  \lnot{x_{i,\lambda}}\lor\lnot{x_{i,\lambda^{'}}}  \right ] \nonumber\\
\left.\right.
\end{align}
\vspace{-15pt}
\begin{align}\label{eq2}
  \left [  \bigland\limits_{2 \leq  i \leq n} \biglor\limits_{1 \leq k < i} l_{i,k} \right ] \land \left [  \bigland\limits_{2 \leq  i  \leq n} \bigland\limits_{1\leq{k}<k^{'}<i}  \lnot{l_{i,k}}\lor\lnot{l_{i,k^{'}}}  \right ]\nonumber\\
  \left.\right.
\end{align}
\vspace{-15pt}
\begin{align}\label{eq3}
  \left [  \bigland\limits_{2 \leq  i \leq n} \biglor\limits_{1 \leq k < i} r_{i,k} \right ] \land \left [  \bigland\limits_{2 \leq  i  \leq n} \bigland\limits_{1\leq{k}<k^{'}<i}  \lnot{r_{i,k}}\lor\lnot{r_{i,k^{'}}}  \right ]
\end{align}
\begin{align}\label{eq4}
    \biglor\limits_{\pi\in\Pi} x_{1,\pi}\\
    \left.\right.
\end{align}
\vspace{-11pt}

We introduce two sets of integer variables to the syntax DAG to add the time bounds on the temporal operators.
These variables are denoted by $a_i$ and $b_i$ (where subscript $i$ is the node identifier).
These two variables are used to store the range of time-step indexes $[j+a_i,j+b_i)$ within which the STL formula $\formula_i$ (valuation of formula $\formula$ at Node $i$) holds $\ltrue$ when evaluated at time-step $t_j$.
In the proposed algorithms, only the temporal operators $\leventually$, $\lglobally$, and $\luntil$ use these $a_i$ and $b_i$.
We add the following constraint to these variables in Eq. \eqref{eq:mui}.


\vspace{-1em}
\begin{align}
    \label{eq:mui}
    \bigland\limits_{1 \leq i \leq n} 0 \leq a_i < b_i \leq |\trace|
\end{align}

Let the propositional formula $\propFormula^{DAG}_n$ be the conjunction of Eqs. \eqref{eq1} to \eqref{eq:mui}.
$\propFormula^{DAG}_n$ encodes the syntax DAG structural constraints for a yet unknown formula of size $n$.
We can reconstruct a syntax DAG from a model $\model$, i.e., a valuation of the propositional variables in $\propFormula^{DAG}_n$, as follows.
1) Label Node $i$ with unique label $\lambda$ such that $\model(x_{i,\lambda})=1$
(if label is $\leventually$, $\lglobally$, or $\luntil$, assign time interval $I=[\model(a_{i}),\model(b_{i}))$ to the operator),
2) set the node $n$ as the root and finally,
3) arrange the nodes of the DAG according to $\model(l_{i,k})$ and $\model(r_{i,k})$.
From this syntax DAG, we can derive an STL formula denoted by $\formula_\model$.

To implement the framework which is introduced in Section \ref{Classifying non-separable interval trajectories framework} in the proposed algorithms, we define two real-valued variables: $\underline{y}^{\trace}_{i,j}$ and $\overline{y}^{\trace}_{i,j}$.
These two variables are equivalent to the robustness margin at the worst-case and the best-case, respectively.
\setlength{\jot}{-5pt}
\begin{align}
    \label{y_all}
    {\underline{y}}^{\trace}_{i,j} &= \min\limits_{\trace\in\intervaltrace}r(\trace,\formula_i,t_j)
    \\
    \label{y_any}
    \overline{y}^{\trace}_{i,j} &= \max\limits_{\trace\in\intervaltrace}r(\trace,\formula_i,t_j)
\end{align}  

In this section, the time length of an interval trajectory $T$ is represented by $|\trace|$, and for the better explanation of the algorithm, we denote the index of the time-step as a position in an interval trajectory; hence, in ${\underline{y}}^{\trace}_{i,j}$ Eq. \eqref{y_all} and $\overline{y}^{\trace}_{i,j}$ Eq. \eqref{y_any}, $i\in\{0,...,n\}$ represents a node in the syntax DAG and $j\in\{0,..,|\trace|-1\}$ is a position in the finite interval trajectory $\intervaltrace$.

Specifically, $\overline{y}^{\trace}_{i,j}$ is used for implementing the robust semantics of negation of an STL formula Eq. \eqref{worst to best}.\par

\vspace{-1em}
\setlength{\jot}{1pt}
\begin{align}\label{worst to best}
    \underline{r}(\intervaltrace,\lnot\formula,t_j)=&\min\limits_{\trace\in\intervaltrace}\left(-r(\trace,\formula,t_j)\right)=\\
    &-\max\limits_{\trace\in\intervaltrace}r(\trace,\formula,t_j)
\end{align}

To implement the semantics of the temporal operators and Boolean connectives, we apply the constraints shown in formulas Eqs. \eqref{eq5} to \eqref{eqlast}.
These constraints are inspired by bounded model checking \cite{Biere2003}.
It should be noted that these constraints are defined similarly for both ${\underline{y}}^{\trace}_{i,j}$, $\overline{y}^{\trace}_{i,j}$. Eq. \eqref{eq5} implements the semantics of the atomic predicates. Eq. \eqref{eq6} implements the semantics of negation. In that formula, if Node $i$ is labeled with $\lnot$ and node $k$ is its left child, then ${\underline{y}}^{\trace}_{i,j}$ is the negation of $\overline{y}^{\trace}_{k,j}$ and its value is $-\overline{y}^{\trace}_{k,j}$.
Similarly, Eq. \eqref{eq7} implements the semantics of disjunction.
In this case, if Node $i$ is labeled with $\lor$, and node $k$ is its left child and node $k^{'}$ is its right child, then ${\underline{y}}^{\trace}_{i,j}$ is equal to the maximum value between ${\underline{y}}^{\trace}_{k,j}$ and ${\underline{y}}^{\trace}_{k^{'},j}$. Eq. \eqref{eqlast} implements the semantics of $\luntil_{[a,b)}$ operator.
In formula \eqref{eqlast}, $a,b\in\{0,1, \dots,|\trace|\}$ denote the starting position and the ending position in which a given STL formula $\formula$ holds $\ltrue$ in interval trajectory $\intervaltrace$.
Similarly, we can define the semantics of other operations: $\leventually_{[a,b)}$, $\lglobally_{[a,b)}$, $\limplies$, $\land$, $\top$.

{\begin{align}\label{eq5}
    \bigland\limits_{1 \leq  i \leq n} \bigland\limits_{\pi\in\Pi} x_{i,\pi}  \limplies \left [  \bigland\limits_{0\leq{j}<{|\trace|}}  \left\{ {{\underline{y}}^{\trace}_{i,j}={\underline{r}}(\intervaltrace,\pi,t_j)} \right\} \right] \nonumber \\
    ~~\left . \right .
\end{align}}\\

\vspace{-25pt}
{\begin{align}\label{eq6}
    \bigland\limits_{\substack{{1 < i < n}\\{1 \le k < i}}} (x_{i,\lnot}\land{l_{i,k}})  \limplies   \bigland\limits_{0\leq{j}<{|\trace|}}
 \left [ {\underline{y}}_{i,j}^{\trace}=-\overline{y}_{k,j}^{\trace}  \right]
\end{align}}\\
\vspace{-25pt}
\begin{align}
    \label{eq7}
    \bigland\limits_{\substack{{1 < i < n}\\{1 \le k,k^{'}< i}}} (x_{i,\lor}\land{l_{i,k}}\land{r_{i,k^{'}}})  \limplies   \bigland\limits_{0\leq{j}<{|\trace|}} \left [{\underline{y}}_{i,j}^{\trace} = \max{({\underline{y}}_{k,j}^{\trace},{\underline{y}}_{k^{'},j}^{\trace}) } \right]
\end{align}\\
\vspace{-25pt}
\begin{align}
    \label{eqlast}
    \bigland\limits_{\substack{{ 1 < i < n}\\{ 1 \le k,k' < i}}} (x_{i,\luntil_{[a,b)}}\land{l_{i,k}}\land{r_{i,k'}})  \limplies~~~~~~~~~~~~~~~~~~~~~~~~~~~~~~~~\\
    \left [ \bigland\limits_{0\leq{j}<{|\trace|}} {\underline{y}}_{i,j}^{\trace}=  \max\limits_{j+a\leq{j'}<{j+b}} \left( \min\left( {\underline{y}}_{k',j'}^{\trace}, \min\limits_{j+a\leq{j''}<{j'}} {\underline{y}}_{k,j''}^{\trace} \right) \right) \right] ~~~~~~~~\\
\end{align}

\vspace{-15pt}
We construct the actual objective function $Y^{\trace}$ equivalent to $\ObjectiveSample(\IntervalSample,\formula_\model)$ in Eq. \eqref{eq:Y}.
We use the negated best-case robustness margin for trajectories labeled as $\negclass$ to take into account the $\lnot$ in Eq. \eqref{eq:objective-sample}.

\begin{align}
  Y^{\trace}
  &:=
  \min\limits_{i=1,..,N_D}
  \begin{cases}
    +\underline{y}^{\trace}_{i,0},
    & \text{if $\class_i=\posclass$}.
    \\
    -\overline{y}^{\trace}_{i,0},
    & \text{if $\class_{i}=\negclass$}.
  \end{cases}
  \label{eq:Y}
\end{align}


	
	

\begin{algorithm}[t] 
	\small
	\Input{
	    Sample $\IntervalSample=\{(\intervaltrace^i,\class_i)\}^{N_D}_{i=1}$ \newline
	    Maximum iteration $\maxiteration \in \nat^{+}$ \newline 
	    Minimum robustness margin $\minrobustness \in \real$ 
	}
	\DontPrintSemicolon
	$n\gets 0$\;
	{
	    \Repeat{$r \geq R$ or $ n > \maxiteration$ }{
	        $n\gets n+1$\;
	    
			Construct formula $\propFormula^{DAG}_{\specDepth}$
			\label{alg:RobustMaxSMT:line:construct}\;
			
			Assign objective function
			$
			Y^{\trace}
			$
			to be maximized
			\label{alg:RobustMaxSMT:line:objective}\;
			
			
			Find model $\model$ using \OptSMT{} solver
			\label{alg:RobustMaxSMT:line:infer}\;
			
			Construct $\formula_\model$ and evaluate $r\gets
			\ObjectiveSample(\IntervalSample, \formula_\model)
			$
			\label{alg:RobustMaxSMT:line:reconstruct}
			\label{alg:RobustMaxSMT:line:evaluate}
		}
    }
	\Return $\formula_\model$\;
	
	
	\caption{\RobustMaxSMT{}}
	\label{alg:RobustMaxSMT}
\end{algorithm}

\paragraph{Uncertainty-Aware Temporal Logic Inference algorithm (\RobustMaxSMT{})}

Algorithm \ref{alg:RobustMaxSMT} shows the procedure of \RobustMaxSMT{}.
We increase the size of the searched formula $n$ (starting from 1) until either of the stopping criteria (described later) triggers.
In each iteration,
we first construct at line \ref{alg:RobustMaxSMT:line:construct} the formula of the structural constraints of the DAG (denoted by $\propFormula^{DAG}_n$).
On top of it, we assign at line \ref{alg:RobustMaxSMT:line:objective} the objective function $Y^{\trace}$, defined in Eq. \eqref{eq:Y}.
We then use \OptSMT{} to get a model $\model$ of $\propFormula^{DAG}_n$ that maximizes $Y^{\trace}$ (line \ref{alg:RobustMaxSMT:line:infer}),
reconstruct the inferred formula
and evaluate the attained objective function value (line \ref{alg:RobustMaxSMT:line:reconstruct}).

The first stopping criteria is triggered when the maximum iteration $\maxiteration \in \nat^{+}$ (given as a parameter) is reached, which produces a formula of maximum size $\maxiteration$.
The second stopping criteria is triggered when the robustness margin threshold $\minrobustness \in \real$ (given as a parameter) is reached.

To solve Problem \ref{non-separable data prob 1}, one can set $\maxiteration$ to the predetermined positive integer described in this problem, and $\minrobustness=+\infty$ in order to ignore the second stopping criteria.
With only the $\maxiteration$ as the stopping criteria, the loop of the algorithm could be ignored and we could directly start at $n = \maxiteration$.
In that case, Algorithm \ref{alg:RobustMaxSMT} returns one of the formula of size $\maxiteration$ that maximizes $\ObjectiveSample(\IntervalSample, \formula)$ (such formula is not unique).

When a finite $\minrobustness$ is specified, Algorithm \ref{alg:RobustMaxSMT} returns an STL formula with size possibly less than $\maxiteration$ but with $\ObjectiveSample(\IntervalSample, \formula_\model) \geq \minrobustness$.
This is particularly useful when the expected size of the STL formula is unknown and $\maxiteration=+\infty$.

\paragraph{Decision Trees over STL Formulas}

\begin{wrapfigure}{r}{0.4\linewidth}
	\centering
	\vspace{-1em}
	\begin{tikzpicture}
	\node (1) at (0, 0) {$\formula_1$};
	\node (2) at (-.8, -.7) {$\formula_2$};
	\node (3) at (.8, -.7) {$\ltrue$};
	\node (4) at (-1.4, -1.6) {$\ltrue$};
	\node (5) at (-.2, -1.6) {$\lfalse$};
	\draw[->]        (1) -- (2);
	\draw[->,dashed] (1) -- (3);
	\draw[->]        (2) -- (4);
	\draw[->,dashed] (2) -- (5);
	\end{tikzpicture}
	\caption{A decision tree over STL formulas}
	\label{fig:DT-example}
	\vspace{-0.5cm}
\end{wrapfigure}
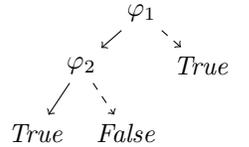

A decision tree over STL formulas is a tree-like structure where all nodes of the tree are labeled by STL formulas.
While the leaf nodes of a decision tree are labeled by either $\ltrue$ or $\lfalse$, the non-leaf nodes are labeled by (non-trivial) STL formulas which represent decisions for predicting the class of a trajectory.
Each inner node leads to two subtrees connected by edges, where the left edge is represented with a solid edge and the right edge with a dashed one.
Figure~\ref{fig:DT-example} depicts a decision tree over STL formulas.

A decision tree $\DT$ over STL formula corresponds to an STL formula $\formula_\DT \coloneqq \bigvee_{\rho\in\mathfrak{R}} \bigwedge_{\formula\in\rho} \formula^\prime$, where $\mathfrak{R}$ is the set of paths that originate in the root node and end in a leaf node labeled with $\ltrue$, $\formula^\prime = \formula$ if it appears before a solid edge in $\rho\in\mathfrak{R}$, and $\formula^\prime = \lnot\formula$ if it appears before a dashed edge in $\rho\in\mathfrak{R}$ (see Figure~\ref{fig:DT-example}).


\begin{algorithm}[th] 
    \small
	\DontPrintSemicolon
	
	\Input{
	    Sample $\IntervalSample=\{(\intervaltrace^i,\class_i)\}^{N_D}_{i=1}$ \newline
	    Maximum iteration $\maxiteration \in \nat^{+}$ \newline
	    Minimum robustness margin $\minrobustness \in \real$
	}
	\BlankLine
	
	$\formula\gets$ Algorithm \ref{alg:RobustMaxSMT} $( \IntervalSample, \minrobustness, \maxiteration )$
	\label{alg:RobustMaxSMTDT:line:infer}\;
	
	Split $\IntervalSample$ into $\IntervalSample^{+}$, $\IntervalSample^{-}$ using $\formula$
	\label{line:split_sample}\;
	
	\uIf{$\stopCriteria(\IntervalSample^{+},\IntervalSample^{-})$\label{line:stop-criteria}}
		{
			\Return{$\leaf(\IntervalSample)$} \label{line:terminate}
		}
	\Else{
		
		$\DT_1\gets$ Algorithm \ref{alg:RobustMaxSMTDT} $(\IntervalSample^{+}, \minrobustness, \maxiteration )$ 
		\label{line:rec_acc}\;
		$\DT_2\gets$ Algorithm \ref{alg:RobustMaxSMTDT} $(\IntervalSample^{-}, \minrobustness, \maxiteration )$ 
		\label{line:rec_rej}\;
		
		\Return{decision tree with root node $\formula$ and subtrees $\DT_{1}$, $\DT_{2}$}
		\;
		}
	\caption{\RobustMaxSMTDT{}}
	\label{alg:RobustMaxSMTDT}
\end{algorithm}

\paragraph{Decision Tree Variant of \RobustMaxSMT{} (\RobustMaxSMTDT{})}

We propose this second method for uncertainty aware STL inference based on decision trees,
outlined by Algorithm \ref{alg:RobustMaxSMTDT}.

First, we infer an STL formula using \RobustMaxSMT{} (line \ref{alg:RobustMaxSMTDT:line:infer}).
Given the inferred formula $\formula$, 
we need a way to split $\IntervalSample$ into two labeled sets $\IntervalSample^{+}$ and $\IntervalSample^{-}$.
Note that the set of $\intervaltrace^i$ with label $\class_{i}=\posclass$ that strongly satisfies $\formula$ and the set of $\intervaltrace^{\hat{i}}$ with the label $\class_{\hat{i}}=\negclass$ that strongly violates $\formula$ do not necessarily partition the  $\IntervalSample$.
As an alternative, we choose to split $\IntervalSample$ with respect to an averaged robustness margin as in Eqs. \eqref{eq:split-accintervalsample} and \eqref{eq:split-rejintervalsample}, in order to have a partition (line \ref{line:split_sample}):
\begin{align}
    \IntervalSample^{+} ={}&
    \left\{
	    \left( \intervaltrace, \class \right) \in \IntervalSample
	    \middle|
	    \vphantom{\frac{\intervaltrace}{2}} \right. \\ & \left.
	    \frac{
	        \underline{r}(\intervaltrace,\formula,t_0)
	        +
	        \overline{r}(\intervaltrace,\formula,t_0)
	    }{2} > 0
	\right\}
	\label{eq:split-accintervalsample}
	\\
	\IntervalSample^{-} ={}&
	\IntervalSample \setminus \IntervalSample^{+}
	\label{eq:split-rejintervalsample}
\end{align}
Based on $\IntervalSample^{+}$ and $\IntervalSample^{-}$, Algorithm \ref{alg:RobustMaxSMTDT} is applied recursively (lines \ref{line:rec_acc} and \ref{line:rec_rej}), if Algorithm \ref{alg:RobustMaxSMTDT} does not terminate at lines \ref{line:stop-criteria} and \ref{line:terminate}.
We define the stopping criteria, $\stopCriteria(\IntervalSample^{+},\IntervalSample^{-})$, for Algorithm \ref{alg:RobustMaxSMTDT} as the following: if $\IntervalSample^{+}=\emptyset$ or $\IntervalSample^{-}=\emptyset$, then $\stopCriteria(\IntervalSample^{+},\IntervalSample^{-})=True$; otherwise, $\stopCriteria(\IntervalSample^{+},\IntervalSample^{-})=False$.
This stopping criteria guarantees the termination of this method,
as the sample size decreases at each split until no split is possible anymore.




\section{Experimental Evaluation }
\label{experimental evaluation}


In this section, we evaluate the performance of the uncertainty-aware proposed algorithms.
In the following, we compare \RobustMaxSMT{} with \SamplingMaxSMT{}, and compare \RobustMaxSMTDT{} with \SamplingMaxSMTDT{}.
We implement all following four algorithms in a C++ toolbox\footnote{\url{https://github.com/cryhot/uaflie}} using Microsoft Z3 \cite{10.5555/1792734.1792766}:
\begin{itemize}
    \item \SamplingMaxSMT{}:
    \MaxSMT{}-based algorithm on finitely many randomly sampled trajectories within the interval trajectories.
    (first baseline);
    
    \item \SamplingMaxSMTDT{}:
    Decision tree variant of \SamplingMaxSMT{} (second baseline);
    
    \item \RobustMaxSMT{}:
    Uncertainty-aware \OptSMT{}-based algorithm on interval trajectories
    (first proposed algorithm);
    
    \item \RobustMaxSMTDT{}:
    Decision tree variant of \RobustMaxSMT{}
    (second proposed algorithm).
\end{itemize}

\subsection{Numerical Evaluation}
For STL inference using \SamplingMaxSMT{} and \SamplingMaxSMTDT{}, we randomly sample a certain number of trajectories from each interval trajectory in the dataset $\IntervalSample$. 
For STL inference using \RobustMaxSMT{} and \RobustMaxSMTDT{}, we directly encode the interval trajectories to the \OptSMT{} solver. 
%

To evaluate the performance of the uncertainty-aware proposed algorithms, we generate 10 datasets. Among these datasets, five datasets are non-separable, and five datasets are separable. In each dataset, both of the sets $P_{unc}$ and $N_{unc}$ contain up to three interval trajectories with the time length up to 10. We use these 10 datasets to infer STL formulas by exploiting algorithms \RobustMaxSMT{} and \RobustMaxSMTDT{}. The inferred STL formulas by \RobustMaxSMT{} are listed in table \ref{formulas tables} with the corresponding optimal worst-case robustness margins denoted by superscript *. Then, we sample 200 trajectories from each interval trajectory in each dataset to infer STL formulas from \SamplingMaxSMT{} and \SamplingMaxSMTDT{} algorithms. First, we evaluate and compare the performance of \RobustMaxSMT{} and \SamplingMaxSMT{}. We choose 1000 seconds for the timeout on each execution. For each individual dataset, we use the same values of parameter \maxiteration~ for both \RobustMaxSMT{} and \SamplingMaxSMT{} when inferring STL formulas for the datasets. The comparison of the execution time of these two algorithms can be seen in Figure \ref{UA-RS}. The results show that the execution time of \RobustMaxSMT{} is at most $1/100$ of the the execution time of \SamplingMaxSMT{} (for a dataset with 800 sampled trajectories in total). Next, we compare \RobustMaxSMTDT{} and \SamplingMaxSMTDT{}. We use same values of parameter \maxiteration{} for both of the \RobustMaxSMT{} and \RobustMaxSMTDT{}. 
The values of parameter \minclassification{} for \SamplingMaxSMTDT{} are within the range $[0,1]$ (see Appendix \ref{ba}). Figure \ref{UA_DT-RS_DT} presents the comparison between the execution time of \RobustMaxSMTDT{} and \SamplingMaxSMTDT{}. \RobustMaxSMTDT{} outperforms \SamplingMaxSMTDT{} in terms of execution time up to four orders of magnitudes (for a dataset with 800 sampled trajectories in total). In the next step, we asses the effectiveness of exploiting decision tree in \RobustMaxSMT{} and \SamplingMaxSMT{}. Figure \ref{RS-RS-DT} presents the comparison between the execution time of \SamplingMaxSMT{} and \SamplingMaxSMTDT{}. The results show that the execution time of \SamplingMaxSMTDT{} is at most $1/14$ of the the execution time of \SamplingMaxSMT{} (for a dataset with 800 sampled trajectories in total). Figure \ref{UA-UA-DT} shows the comparison between the execution time of \RobustMaxSMT{} and \RobustMaxSMTDT{}. \RobustMaxSMTDT{} outperforms \RobustMaxSMT{} by inferring STL formulas faster. The execution time of \RobustMaxSMTDT{} is at most $1/88$ of the execution time of \RobustMaxSMT{}.
\begin{figure}[h]
     \begin{subfigure}[b]{0.22\textwidth}
         \includegraphics[scale=0.17]{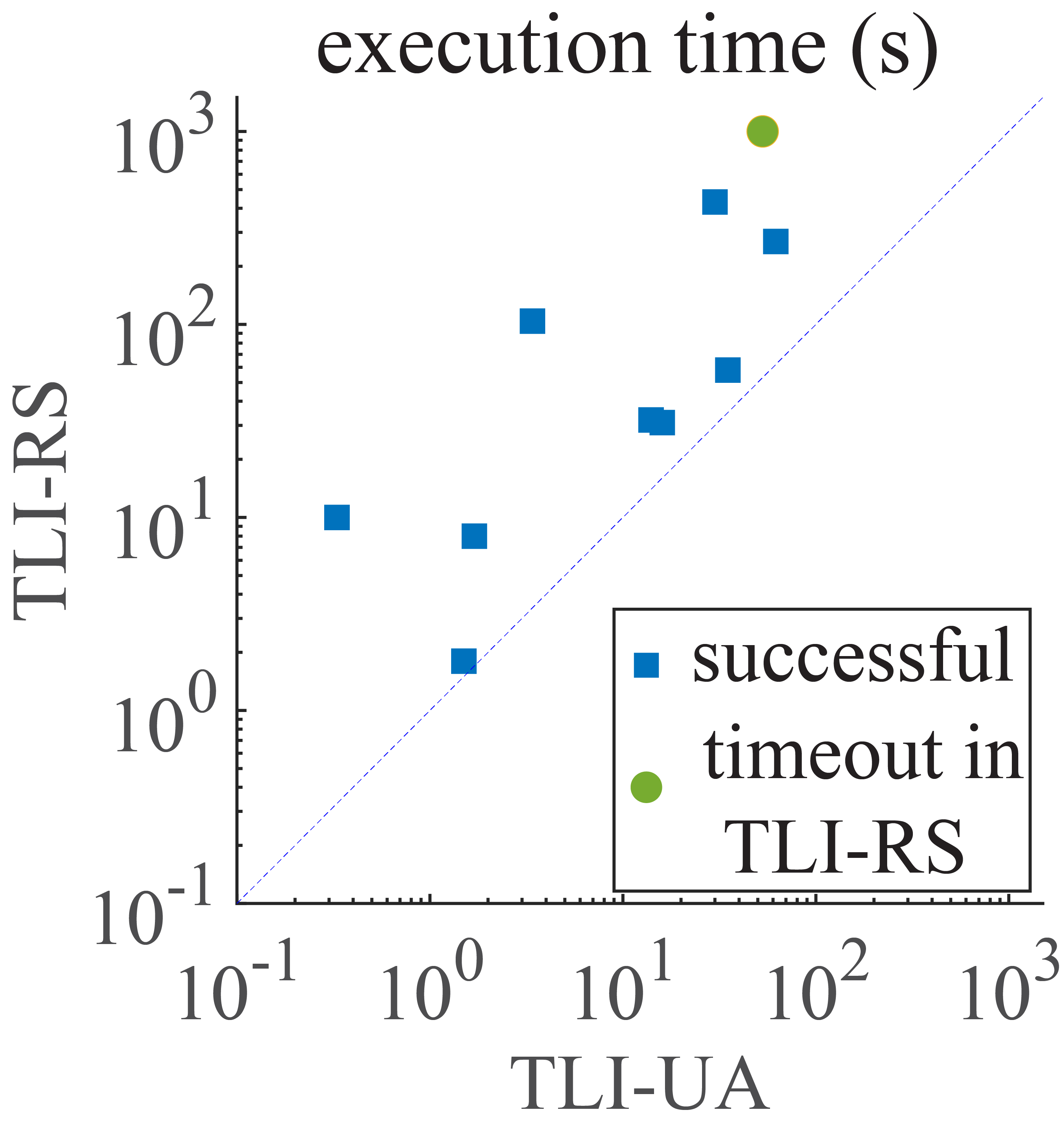}
      \caption{}
        \label{UA-RS}
    \end{subfigure}
    \hfill
   \begin{subfigure}[b]{0.22\textwidth}
      \includegraphics[scale=0.17]{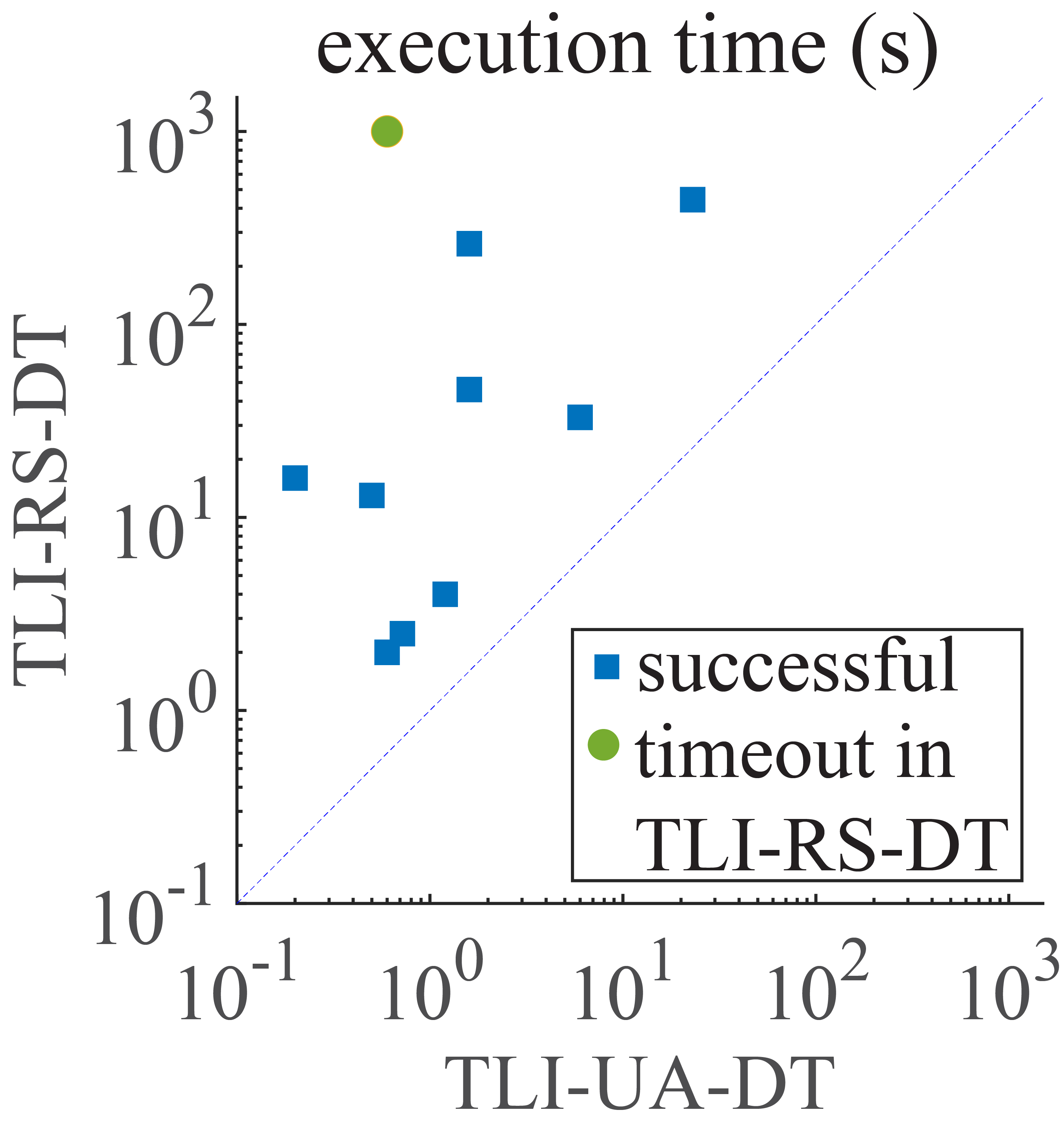}
       \caption{}
        \label{UA_DT-RS_DT}
   \end{subfigure}
   \caption{The comparison of the execution time between: a) \RobustMaxSMT{} and \SamplingMaxSMT{}, where the execution time of \RobustMaxSMT{} is at most $1/100$ of the execution time of \SamplingMaxSMT{} (for a dataset with 800 sampled trajectories in total), and b) between \RobustMaxSMTDT{ and \SamplingMaxSMTDT{}}, where \RobustMaxSMTDT{} reduces the execution time up to four orders of magnitudes (for a dataset with 800 sampled trajectories in total). }\end{figure}
   
   \begin{figure}[h]
     \begin{subfigure}[b]{0.22\textwidth}
         \includegraphics[scale=0.17]{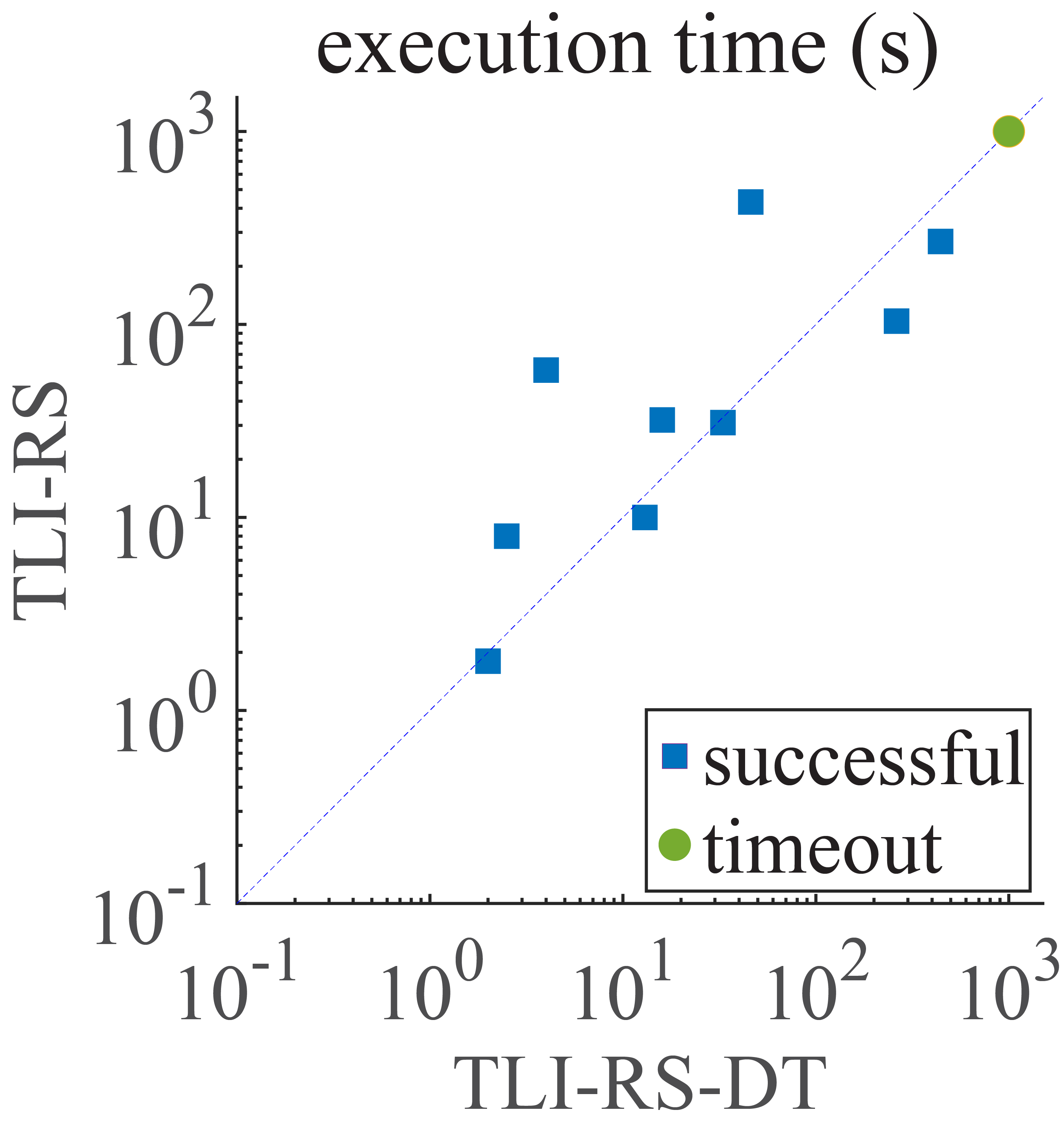}
      \caption{}
        \label{RS-RS-DT}
    \end{subfigure}
    \hfill
   \begin{subfigure}[b]{0.22\textwidth}
      \includegraphics[scale=0.17]{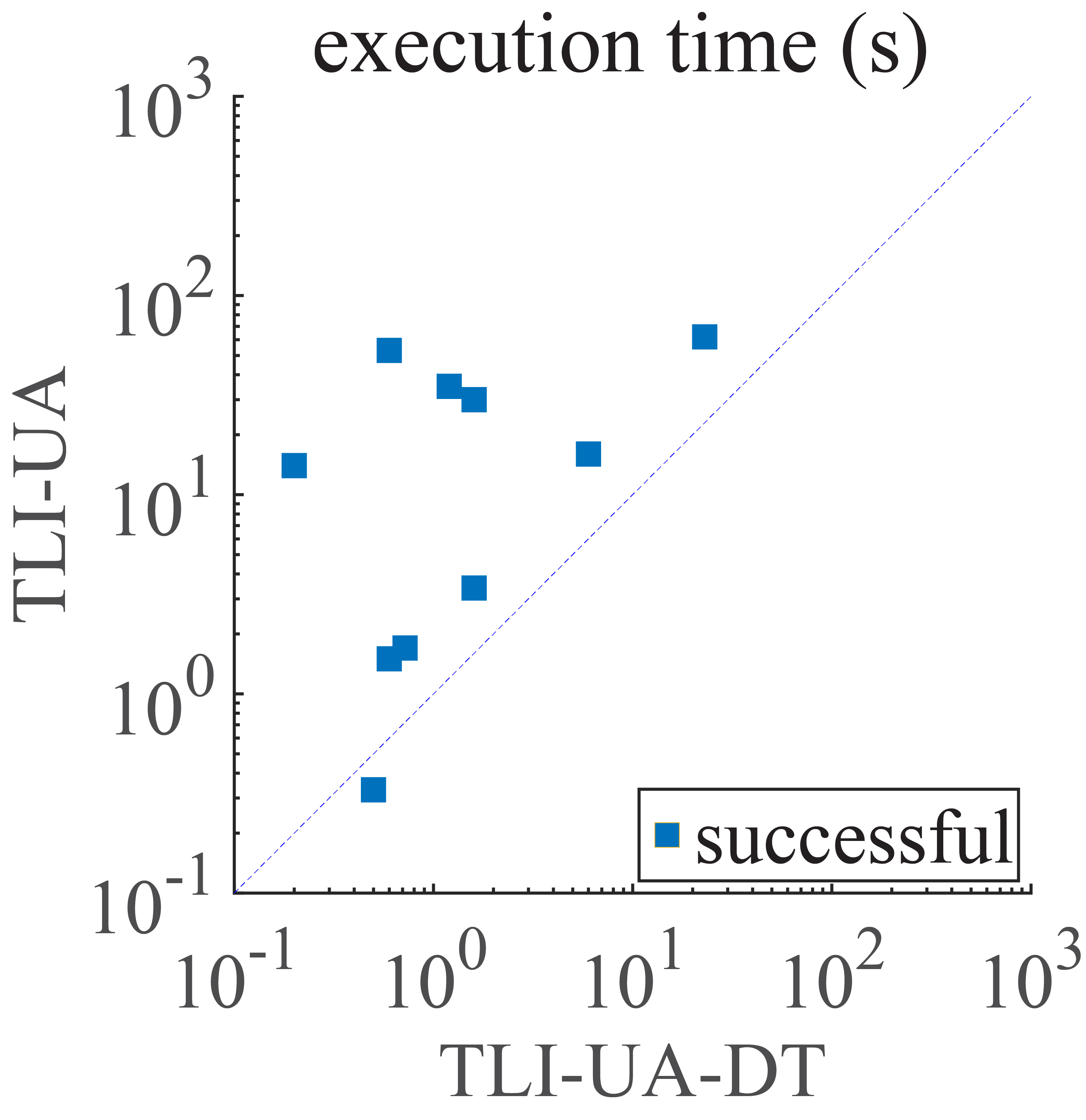}
       \caption{}
        \label{UA-UA-DT}
   \end{subfigure}
   \caption{The effectiveness of adding decision tree to \SamplingMaxSMT{} and \RobustMaxSMT{}: a) The execution time of \SamplingMaxSMTDT{} is at most $1/14$ of the execution time of \SamplingMaxSMT{} (for a dataset with 800 sampled trajectories in total). b) The execution time of \RobustMaxSMTDT{} is at most $1/88$ of the execution time of \RobustMaxSMT{}.}\end{figure}

    
%


\vspace{-10pt}
\begin{table}
\renewcommand{\arraystretch}{1.5}
\addtolength{\tabcolsep}{-5pt}
\begin{tabular}{c| c c} \toprule
    data type~~~~&{inferred STL  } & {${\underline{r}^{*}(\intervaltrace, \formula,t_0)}$}   \\
    
   &{formulas by \RobustMaxSMT{}} ($\formula$) & \\\midrule
       &   $\leventually_{[6,10)}(x^1+x^2<20)$  & 0  \\\cline{2-3}

    separable~~~& $\leventually_{[3,10)}(x^1-x^2<0)$  & 0 \\
    datasets          & $\limplies(x^1-x^2<0)$   &    \\\cline{2-3}
                  &$(x^1>10)\vee(x^2<-12)$     & 0    \\\cline{2-3}

                  &$\leventually_{[3,4)}(x^1+x^2>15.5)$  & 2.5    \\ \cline{2-3}
                  &$\lglobally(x^1>2.5)$               & 1.6    \\\midrule
                &$({x^1}<0.5)\luntil_{[8,10)}(x^1>0.5)$ & -4.5  \\\cline{2-3}
 non-separable~~~& $(x^1<10)\wedge({x^1-x^2>8.9)}$   & -10 \\\cline{2-3}
           datasets  & $\lglobally_{[9,10)}(x^1-x^2>4.5)$ & -5     \\\cline{2-3}
                  & $\leventually_{[0,2)}(x^1+x^2>8)$ & -4\\
                  & $\limplies(x^1+x^2>8)$               & \\\cline{2-3}
                  &$\leventually_{[1,10)}(\lnot{x^2<4.5)}$ & -0.5  \\  
    \bottomrule
\end{tabular}\par

\caption{The inferred STL formulas by \RobustMaxSMT{}, and the corresponding worst-case robustness margins for each of the 10 generated datasets. }\label{formulas tables}
\end{table}
\vspace{-10pt}
\subsection{Strategy Inference of Pusher-Robot Scenario}
 In this case-study, we infer uncertainty-aware STL formulas to describe the behavior of the interval trajectories of a Pusher-robot. The interval trajectories are generated by policies learned from reinforcement learning (RL) using \textit{model-based reinforcement learning} (MBRL) algorithm \cite{Nagabandi}. The intervals represent the uncertainties associated with the policies and the environment. This robot consists of two components denoted as the  \textquotedblleft{forearm}\textquotedblright~ and the \textquotedblleft{upper arm}\textquotedblright~ (Figure \ref{pusher robot shematic}). In this paper, we investigate four different strategies of this Pusher-robot: 1) Tap the ball toward the wall. 2) Tap the ball, rotate around, and stop the ball. 3) Tap the ball, stop the ball with the upper arm. 4) Bounce the ball off the wall. 
We infer an STL formula for each strategy versus the other three strategies based on the change of the speed ($m/{s}$) of the ball during performing the current strategy which is denoted by $x^1$. For the dataset, $P_{unc}$ includes the interval trajectories of the current strategy, and $N_{unc}$ includes the interval trajectories of the other three strategies. The inferred STL formulas are presented in Table \ref{strategies}. If $x^1<0$, the contact (the contact between the robot and the wall or contact between the ball and the ball) absorbs momentum of the ball, and if $x^1>0$, then the contact adds momentum to the ball.\par
 
\begin{figure}[h]
    \centering
        \includegraphics[scale=0.15]{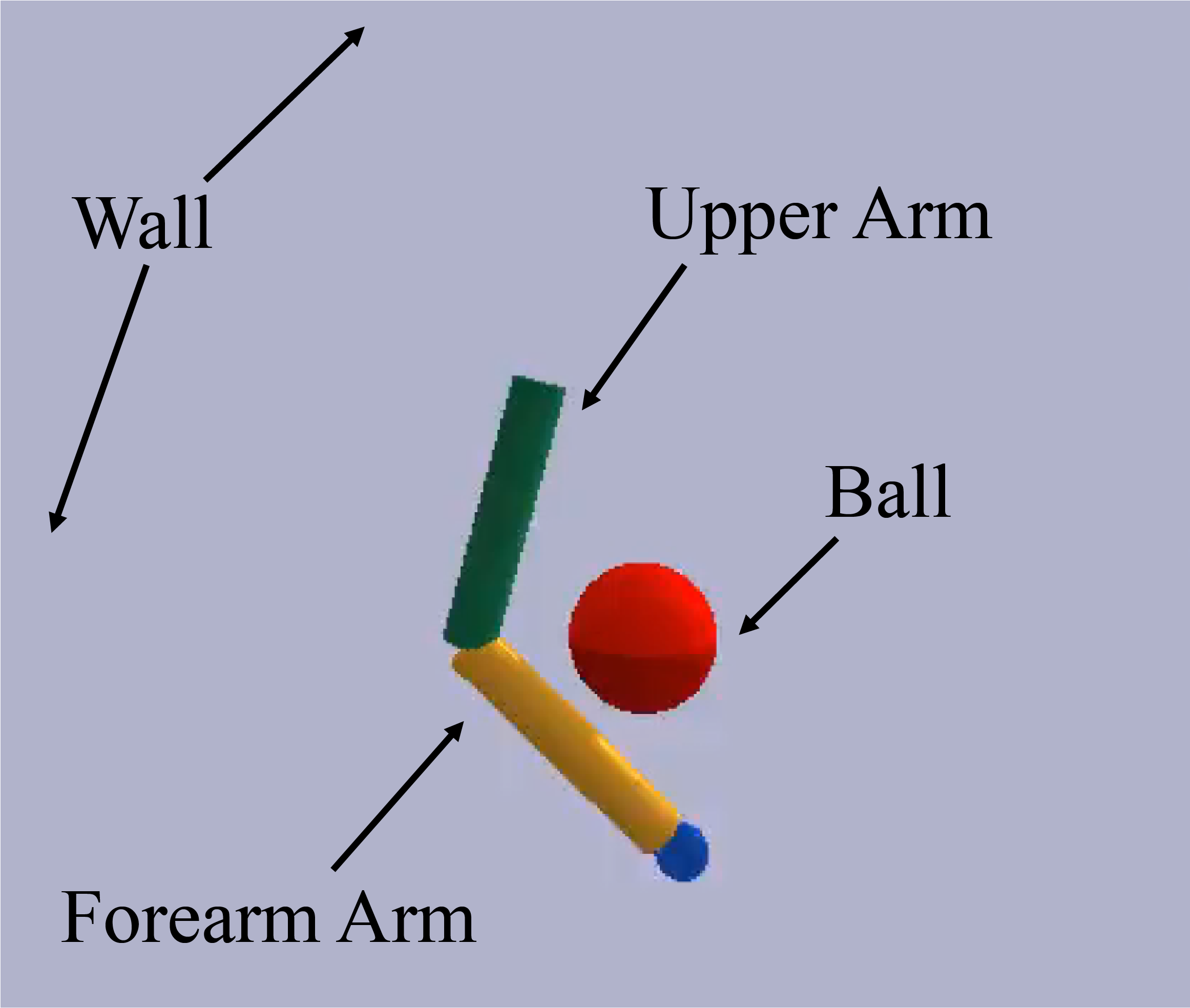} 
       \caption{The schematic of the Pusher-robot with the two components denoted as the forearm and the upper arm.}
        \label{pusher robot shematic}
\end{figure}

\begin{table}
\renewcommand{\arraystretch}{1.5}
\addtolength{\tabcolsep}{-2pt}
\begin{tabular}{c| c c} \toprule
   {strategy} & inferred STL & {${\underline{r}^{*}(\intervaltrace, \formula,t_0)}$}\\ 
    & formula by \RobustMaxSMT{} (\formula)&\\ \hline
   \multirow{ 2}{*}{strategy 1} & $(x^1>-0.232)$ &  \multirow{ 2}{*}{-0.848} \\
    &$\luntil_{[2,4)}(x^1>-0.232)$ &\\\cline{2-3}
     \multirow{ 2}{*}{strategy 2}   & $(x^1>0.21)$ &  \multirow{ 2}{*}{-0.4} \\
      & $\luntil_{[0,3)}(\lglobally_{[1,3)}(x^1>0.21))$ & \\\cline{2-3}
    {strategy 3} & $\lglobally_{[1,4)}(x^1>0.266)$ & -0.35     \\\cline{2-3}
 \multirow{ 2}{*}{strategy 4}  & $(x^1>0.231)$&   \multirow{ 2}{*}{-0.385}\\
  & $\luntil_{[4,5)}(x^1>0.231)$ &\\
   \bottomrule
\end{tabular}\par
\caption{The inferred STL formulas by \RobustMaxSMT{} for strategies 1 to 4, and the corresponding worst-case robustness for feature $x^1$ (the change in the speed of the ball during performing the current strategy).}\label{strategies}
\end{table}\par



The interpretations of the inferred STL formulas for the strategies, respectively from 1 to 4, are: 1) The change in the speed of the ball is greater than -0.232 $m/{s}$, until sometime from time-step 2 to time-step 4, the change in the speed of the ball is greater than -0.232 $m/{s}$ (transition from losing momentum to gaining momentum). 2) The change in the speed of the ball is greater than 0.21  $m/{s}$, until some time from time-step 0 to time-step 3, the change of the speed of the ball is always greater than 0.21 $m/{s}$ (gaining momentum from time-step 0 to time-step 3). 3) Some time from time-step 1 to time-step 4, the change of the speed of the ball is always greater than 0.266 $m/{s}$ and gaining momentum. 4) The change in the speed of the ball is greater than 0.231 $m/{s}$, until sometime from time-step 4 to time-step 5, the change in the speed of the ball is greater than 0.231 $m/{s}$ (gaining momentum).

\section{Conclusion}
In this paper, we proposed two uncertainty-aware STL inference algorithms. Our results showed that uncertainty-aware STL inference expedites the inference process in the presence of uncertainties. Exploiting uncertainty-aware STL inference to enhance RL is one future research direction (in a similar manner as in \cite{ICAPS20paper116}). Moreover, we aim to develop the proposed learning algorithms to uncertainty-aware \textit{graph temporal logic} (GTL) inference, where we can infer spatial temporal properties from data with uncertainties\cite{Xugti}.
\label{conclusion}


  \section{Acknowledgment}
 The authors thank Dr. Rebecca Russell and the entire ALPACA team for their collaboration. This material is based upon work supported by the Defense Advanced Research Projects Agency (DARPA) under Contract No. HR001120C0032. Any opinions, findings and conclusions or recommendations expressed in this material are those of the author(s) and do not necessarily reflect the views of DARPA.

\titleformat{\section}{\centering\normalfont\scshape}{\appendixname~\thesection }{0em}{~}
\bibliographystyle{IEEEtran}
\bibliography{zherefclean_submit}

\titleformat{\section}{\centering\normalfont\scshape}{\appendixname~\thesection }{0em}{~}

\clearpage
\appendices
 \section{supplementary Mathematical materials}\label{appendix:math}

\subsection{Proof of Theorem \ref{thm:perfect-classification-2traj}}
\label{apdx:prf:perfect-classification-2traj}
For two given disjoint intervals, $[\underline{a},\overline{a}]$ and $[\underline{b},\overline{b}]$, we know that $\underline{a}\leq\overline{a}$ and $\underline{b}\leq\overline{b}$. If $\overline{a}<\underline{b}$, then we define $[\underline{a},\overline{a}]<[\underline{b},\overline{b}]$; or if $\overline{b}<\underline{a}$, then $[\underline{a},\overline{a}]>[\underline{b},\overline{b}]$.\par
For two interval trajectories,
$\intervaltrace^i$ with label $\class_i=\posclass$ and 
$\intervaltrace^{\Tilde{i}}$ with label $\class_{\Tilde{i}}=\negclass$,
if there exists one time-step $t_j$ and one dimension $k$ such that $\intervaltrace[j][k]^{\Tilde{i}}\cap\intervaltrace[j][k]^i=\emptyset$;
then, either $\intervaltrace[j][k]^i>\intervaltrace[j][k]^{\Tilde{i}}$ or $\intervaltrace[j][k]^i<\intervaltrace[j][k]^{\Tilde{i}}$. Without loss of generality, we take $\intervaltrace[j][k]^i>\intervaltrace[j][k]^{\Tilde{i}}$. Moreover, we know that $\overline{\trace}_j^k\in\intervaltrace[j][k]^{\Tilde{i}}$ and $\underline{\trace}_j^k\in\intervaltrace[j][k]^{{i}}$ are real numbers.

If we represent $\overline{\trace}_j^k\in\intervaltrace[j][k]^{\Tilde{i}}$ by $\Tilde{d}$, and represent $\underline{\trace}_j^k\in\intervaltrace[j][k]^{{i}}$ by $d$,
we know for any two real numbers $\{d,\Tilde{d}~|~ d>\Tilde{d}\}$, there is a real value $\delta=\frac{\Tilde{d}+d}{2}$ such that $d>\frac{\Tilde{d}+d}{2}>\Tilde{d}$. Then, we can conclude that $\intervaltrace[j][k]^{i}>\delta$.
Therefore, there exists at least one STL formula in the form of $\leventually_{[t_j,t_{j+1})}(x^k>\delta)$ that perfectly classifies the two interval trajectories.

\subsection{Proof of Theorem \ref{thm:perfect-classification-sample}}
\label{apdx:prf:perfect-classification-sample}
Given a labeled set of interval trajectories $\IntervalSample=\{(\intervaltrace^i,\class_i)\}^{N_D}_{i=1}$, we represent interval trajectories with label  
$\class_{{i}}=\posclass$ by $\intervaltrace^i$ and represent interval trajectories with label $\class_{{i'}}=\negclass$ by $\intervaltrace^{{i'}}$.

By relying on Theorem \ref{thm:perfect-classification-2traj}, if all the pairs of interval trajectories $\intervaltrace^i$ and $\intervaltrace^{i'}$ are separable, then a formula $\formula_{ii'}$ can be found that is strongly satisfied by interval trajectories $\intervaltrace^i$ with the label $\class_i=\posclass$ and is strongly violated by interval trajectories $\intervaltrace^{i'}$.
This formula can be in the form of $\formula=\biglor_{(\intervaltrace^i, \posclass) \in \IntervalSample} \bigland_{(\intervaltrace^{i'}, \negclass) \in \IntervalSample}\formula_{ii'}$.
This formula can strongly classify the two sets of interval trajectories since $\formula$ is strongly satisfied by all $\intervaltrace^i$ and strongly violated by all $\intervaltrace^{i'}$.


\section{Baseline Algorithms}\label{ba}

Algorithm \ref{alg:MaxSMT} and Algorithm \ref{alg:MaxSMTDT}
present the incomplete procedures for the two baseline algorithms \SamplingMaxSMT{} and \SamplingMaxSMTDT{}, respectively.
The complete procedures would include the following first step:
construct $\Sample$ from finitely  many randomly sampled  trajectories  within the  interval trajectories $\IntervalSample$.

\paragraph{Baseline Temporal Logic Inference algorithm (\SamplingMaxSMT{})}

\begin{algorithm}[th] 
	\small
	\Input{
	    Sample $\Sample=\{({\trace}^i,\class_i)\}^{N_D}_{i=1}$ \newline
	    Minimum classification $\minclassification \in [0,1]$ \newline 
	    Maximum iteration $\maxiteration \in \nat^{+}$ 
	}
	\DontPrintSemicolon
	$n\gets 0$\;
	{
	    \Repeat{$s \geq \minclassification$ or $ n > \maxiteration$ }{
	        $n\gets n+1$\;
	    
			Construct formula $\propFormula^{DAG}_{\specDepth} \land \propFormula^{\Sample}_{\specDepth}$
			\;
			
			Assign weights to soft constraints 
			\;
			
			
			Find model $v$ using \MaxSMT{} solver
			(or \SMT{} solver if $\minclassification=1$)
			\;
			
			Construct $\formula_v$ and evaluate $s\gets
			\frac{
			    \left\{ ({\trace}^i,\class_i) \in \Sample \middle| {\trace}^i \satisfies \formula_v \right\}
			}{|\Sample|}
			$
			
		}
    }
	\Return $\formula_v$\;
	
	
	\caption{TLI}
	\label{alg:MaxSMT}
\end{algorithm}

Algorithm \ref{alg:MaxSMT} presents similarities with Algorithm \ref{alg:RobustMaxSMT} in its structure.
We cover the key differences of Algorithm \ref{alg:MaxSMT} here.

We define propositional formulas  $\semanticConst{\specDepth}{\trace}$ for each trajectories $\trace$ that tracks the valuation of the STL formula encoded by $\propFormula^{DAG}_{\specDepth}$ on $\trace$.
These formulas are built using variables $\y{i}{j}{\trace}$, where $i\in\{1,\ldots,\specDepth\}$ and $j\in\{1,\ldots,\abs{\zeta}-1\}$, that corresponds to the value of $\valFuncPos{\formula_i}{\trace}{j}$ ($\formula_i$ is the STL formula rooted at Node~$i$).

We now define the constraint that ensure consistency with the sample $\propFormula^{\Sample}_{\specDepth}$:
\begin{align}\label{eq:consistency-constraints}
    \propFormula^{\Sample}_{\specDepth} =
    \bigland\limits_{(\trace, \class)\in\Traces} \semanticConst{\specDepth}{\trace} \land
    \bigland\limits_{(\trace,\posclass)\in\Sample} \y{\specDepth}{0}{\trace} \land
    \bigland\limits_{(\trace,\negclass)\in\Sample} \lnot\y{\specDepth}{0}{\trace}
\end{align}
\noindent Each of the $\y{\specDepth}{0}{\trace}$ and $\lnot\y{\specDepth}{0}{\trace}$ are soft constraints (for sake of simplicity, each one is attributed a weight of $1$); all the other constraints are hard constraints.

The previously defined soft constraints aims at correctly classifying a maximum number of trajectories in $\Sample$.
We introduce a new stopping criteria that is triggered
when the percentage of correctly classified trajectories exceeds a given threshold $\minclassification \in [0,1]$.

\paragraph{Decision Tree Variant of \SamplingMaxSMT{} (\SamplingMaxSMTDT{})}

\begin{algorithm}[th] 
    \small
	\DontPrintSemicolon
	
	\Input{
	    Sample $\Sample=\{({\trace}^i,\class_i)\}^{N_D}_{i=1}$ \newline
	    Minimum classification $\minclassification \in [0,1]$ \newline 
	    Maximum iteration $\maxiteration \in \nat^{+}$ 
	}
	\BlankLine
	
	$\formula\gets$ Algorithm \ref{alg:MaxSMT} $( \Sample, \minclassification, \maxiteration )$
	\label{alg:MaxSMTDT:line:infer}\;
	
	Split $\Sample$ into $\Sample^{+}$, $\Sample^{-}$ using $\formula$
	\label{alg:MaxSMTDT:line:split_sample}\;
	
	\uIf{$\stopCriteria(\Sample^{+},\Sample^{-})$\label{line:stop-criteria}}
		{
			\Return{$\leaf(\Sample)$}
		}
	\Else{
		
		$\DT_1\gets$ Algorithm \ref{alg:MaxSMTDT} $(\Sample^{+}, \minclassification, \maxiteration )$
		\label{alg:MaxSMTDT:line:rec_acc}\;
		
		$\DT_2\gets$ Algorithm \ref{alg:MaxSMTDT} $(\Sample^{-}, \minclassification, \maxiteration )$
		\label{alg:MaxSMTDT:line:rec_rej}\;
		
		\Return{decision tree with root node $\formula$ and subtrees $\DT_{1}$, $\DT_{2}$}
		\;
		}
	\caption{TLI-DT}
	\label{alg:MaxSMTDT}
\end{algorithm}

Algorithm \ref{alg:MaxSMTDT} is very similar to Algorithm \ref{alg:RobustMaxSMTDT}.
We split $\Sample$ into accepted and rejected samples as in \eqref{eq:split-accsample} and \eqref{eq:split-rejsample} at line \ref{alg:MaxSMTDT:line:split_sample}:
\begin{align}
    \Sample^{+} ={}&
    \left\{
	    \left( {\trace}, \class \right) \in \Sample
	    \middle|
	    {\trace} \models \formula
	\right\}
	\label{eq:split-accsample}
	\\
	\Sample^{-} ={}&
    \left\{
	    \left( {\trace}, \class \right) \in \Sample
	    \middle|
	    {\trace} \not\models \formula
	\right\}
	\label{eq:split-rejsample}
\end{align}

\end{document}